\documentclass[11pt]{article}

\title{Better Private Algorithms for Correlation Clustering}
\usepackage{times}
\usepackage{xspace} 
\usepackage{geometry}
\usepackage{thmtools} 
\usepackage{thm-restate}

\usepackage{tikz}
\usepackage{hyperref}  
\hypersetup{colorlinks=true,citecolor=blue,linkcolor=red}

\author{Daogao Liu \thanks{University of Washington. Email: \texttt{dgliu@uw.edu}}}

\usepackage{amsmath}
\usepackage{amssymb}
\usepackage{mathtools}
\usepackage{amsthm}
\usepackage{natbib}

\usepackage{algorithm}
\usepackage{algorithmic}

\theoremstyle{plain}
\newtheorem{theorem}{Theorem}[section]
\newtheorem{lemma}[theorem]{Lemma}
\newtheorem{proposition}[theorem]{Proposition}
 \newtheorem{fact}[theorem]{Fact}
 \newtheorem{claim}[theorem]{Claim}

\theoremstyle{definition}
\newtheorem{definition}[theorem]{Definition}

\theoremstyle{remark}

\def\gC{{\mathcal{C}}}

\def\gP{{\mathcal{P}}}

\DeclareMathOperator{\cut}{cut}

\newcommand{\C}{\mathcal{C}}
\newcommand{\E}{\mathbb{E}}

\newcommand{\dis}{\mathrm{dis}}
\newcommand{\ALG}{\mathrm{ALG}}

\newcommand{\poly}{\mathrm{poly}}

\newcommand{\hesi}{\ensuremath{\textsf{hesitant}}\xspace}
\newcommand{\good}{\mathrm{good}}

\newcommand{\PJudgeG}{\mathrm{PJudgeGood}}

\newcommand{\Cleanup}{\mathrm{CleanUp}}

\newcommand{\cost}{\mathrm{cost}}

\newcommand{\Ind}{\boldsymbol{1}}

\newcommand{\Event}{\mathrm{Evt}}

\newcommand{\RL}{\mathrm{RL}}
\newcommand{\R}{\mathbb{R}}
\newcommand{\D}{\mathcal{D}}

\newcommand{\MinDis}{\mathrm{MinDis}}
\newcommand{\Lap}{\mathrm{Lap}}
\newcommand{\OPT}{\mathrm{OPT}}

\begin{document}
\maketitle
\begin{abstract}
    In machine learning, correlation clustering is an important problem whose goal is to partition the individuals into groups that correlate with their pairwise similarities as much as possible.
In this work, we revisit the correlation clustering under the differential privacy constraints.
Particularly, we improve previous results and achieve an $\Tilde{O}(n^{1.5})$ additive error compared to the optimal cost in expectation on general graphs. 
As for unweighted complete graphs, we improve the results further and propose a more involved algorithm which achieves $\Tilde{O}(n \sqrt{\Delta^*})$ additive error, where $\Delta^*$ is the maximum degrees of positive edges among all nodes.
\end{abstract}

\section{Introduction}

Correlation clustering, introduced in the seminal work of \cite{bbc04}, is a widely used  algorithm in machine learning.
\iffalse
Clustering, which is the task of dividing a population or set of data points into different groups based on their similarity, is a widely machine learning primitive.
However, many famous algorithms like $k$-means, $k$-sum and $k$-center need a specific objective number of clusters.
Finding a good $k$ can be a difficult task and cannot be easily determined.
Motivated by this, \cite{bbc04} proposed a new clustering problem named by correlation clustering.
\fi
In this problem, we are given a graph where each edge is labeled either positive or negative, and has a non-negative weight.
These weights along with their signs measure the magnitude of similarity or dissimilarity between two nodes.
The correlation clustering problem asks to find a partition $\C_1,\cdots,\C_k$ of the node set $V$, such that all positive-labeled edges connect nodes in the same cluster and all negative-labeled edges connect nodes in different clusters.
%But it is common that the information is inconsistent and one can not find such a perfect clustering.
However, as the problem is NP-hard, one can not always find such a perfect clustering, and need to settle for an approximate solution.
There are two widely studied notions of approximate solutions.
In Maximum Agreement (MaxArg) problem, we want to maximize the weight of positive edges inside the clusters plus the weight of negative edges between the clusters.
In Minimum Disagreement (MinDis) problem, we aim to get a clustering which minimizes the total weight of violated edges, which is defined as the weight of negative edges inside the clusters plus the weight of positive edges between the clusters.
As getting a constant approximation to MaxArg problem is much easier and less interesting, we focus on MinDis problem in this work, like most of the previous papers.

In many applications, the underlying graph can contain sensitive information about individuals; think of social networks for example. 
In recent years, privacy has become an important consideration for learning algorithms.
In particular, differential privacy (DP), introduced in the seminal work of \cite{dmns06}, has become de facto standard notion of privacy for machine learning problems.
These considerations motivated \cite{bek21} to initiate the study of correlation clustering problem under DP constraints.
%They considered DP correlation clustering algorithms with mixed multiplicative and additive guarantees, i.e., they want algorithms that report clustering with $\MinDis\leq \alpha \OPT+\beta$.
As they observed, the exponential mechanism (\cite{MT07}), one of the classic mechanisms in DP, can achieve an additive error of $O(\frac{n}{\epsilon}\log n)$.
However, it takes exponential time and thus is inefficient.
Further, they also showed a lowerbound of $\Omega(n/\epsilon)$ on the additive error.
On the other hand, for general graph, they proposed an {\em efficient} polynomial time $(\epsilon,\delta)$-DP algorithm that achieves an additive error of $O(n^{1.75}/\epsilon)$.
The main focus of this work is to design algorithms with better additive errors.

% *Correlation clustering is well-studied problem in machine learning.
% *Introduce the problem definition. 
% *Like all the previous papers, as agreement is easy, we focus on the disagreement minimizaation problem.
% *In many applications privacy is a concern.
% *So, BUn et al initated the study of DP correlation.
% * They showed a lowerbound o(n)
% * They showed an upper bound of O().
% * Furhtermore, exponential mechanism obtains O(n log n), although is not polytime.
% * In this paper we improve the results of Bun et al.

\subsection{Our Contributions}
In this paper, we improve the results of \cite{bek21}. 
%with a more sophisticated analysis and achieve an additive error $O(n^{1.5}/\sqrt{\epsilon})$.
For general weighted graphs we obtain the following result:

\begin{theorem}[Informal]
\label{thm:informal_weighted}
For $0<\epsilon<1/2$ and $0<\delta<1/2$,
given a graph $G$ with weighted edges, there is an efficient $(\epsilon,\delta)$-DP algorithm with
\begin{align*}
    \MinDis\leq O(\log n)\cdot \OPT+\Tilde{O}(n^{1.5}/\sqrt{\epsilon}).
\end{align*}
\end{theorem}

%This improved over by a factor of $O(n^{1/4})$ compared to the result of \cite{bek21}.

For unweighted complete graphs (each edge has unit weight), we show an improved bound:

\begin{theorem}[Informal]
Given an unweighted graph $G$, there is an efficient $(\epsilon,\delta)$-DP algorithm with
\begin{align*}
    \MinDis\leq O(1) \cdot \OPT+\Tilde{O}(n\sqrt{\Delta^*}/\epsilon),
\end{align*}
where $\Delta^*$ is the maximum positive degree of nodes in graph.
\end{theorem}

Both these results improve the  additive error of \cite{bek21} by a factor of at least  $O(n^{1/4})$ in the worst case.
On the other hand, the multiplicative errors match the best non-private algorithms up to $O(1)$ terms. 
Moreover, when the maximum positive degree is $o(n)$, using Theorem 2, we get significantly improved additive errors.

%In comparison, the additive error of \cite{bek21} is $\Tilde{O}(n^{1.75})$, which is worse than our error ($O(n^{1.5})$ when $\Delta^*=n$).
%When the maximum positive degree is $o(n)$, we can get a better additive error than Theorem~\ref{thm:informal_weighted} for unweighted graph.

\subsection{Our Techniques}
For the general (weighted) version, our algorithm follows a similar outline as in \cite{bek21}, and our improvement comes from a more delicate analysis.
At a high level, \cite{bek21} use DP algorithm to release a synthetic graph $H$ which approximates the original graph $G$ in terms of cut distance within a factor of $O(\sqrt{mn})$, where $m$ is the total weights of all edges in the graph.
Then they do a post-processing on $H$ to find a clustering consisting of at most $k=O(n^{1/4})$ partitions. 
They argue that the total number of disagreements (and agreements) of a fixed clustering consisting of $k$ clusters on $G$ and $H$ differ by at most $k$ times the respective cut distance bound, thus leading to an additive error of $O(n^{1/4}) \cdot O(\sqrt{mn})$ which is at most 
$O(n^{1.75})$ if $m=O(n^2)$.
Using a simple probabilistic argument we show that the factor $k$ is not necessary, and a constant times the respective cut distance bound is good enough to bound the total number of disagreements (and agreements).
This leads to an improved bound of $O(\sqrt{mn})$ on the additive error, and specifically $O(n^{1.5})$ when $m=O(n^2)$.

On the other hand, our algorithm for the unweighted disagreement minimization on complete graphs follows a completely different approach.
We present a private algorithm that achieves an $\Tilde{O}(n\sqrt{\Delta^*})$ additive error, where $\Delta^*$ is the maximum positive degree among all nodes in the graph.
Note that achieving an additive error of $O(n\Delta^*)$ is trivial by simply outputting all nodes as singletons, but getting  $\sqrt{\Delta^*}$ is non-trivial and generalizes the previous result for weighted graphs.

%proposes based on the clean sets.
Our algorithm works as follows.
Say a node in the graph is {\em good} with respect to a set, if the neighborhood of the node overlaps with the set well, and
a set is {\em clean}, if all nodes in it are good with respect to the set.
We process nodes one-by-one and in each iteration, we choose one arbitrary node $v$ as a {\em pivot}.
If the positive degree of $v$ is small, we can output $v$ as a singleton directly.
Otherwise, we find the set $B$ of nodes which are $\lambda$-good w.r.t. the neighborhood $N^+(v)$ of $v$.
If $|B|$ is a constant fraction smaller than the size of $|N^+(v)|$, say $|B|<0.9 |N^+(v)|$, we output $v$ as a singleton;
Else, we keep $\min\{|B|,2|N^+(v)|\}$ nodes in $B$ and delete the remaining, and we find the set $D$ from the remaining nodes $V\setminus B$ which are $4\lambda$-good w.r.t. $B$.
Similarly, we keep $\min\{|D|,2|B|\}$ nodes in $D$ and delete others, and output $D\cup B$ as a cluster.
Our algorithm is loosely inspired by the constant approximation algorithm for the correlation clustering problem due to \cite{bbc04}, in particular, the notions of good nodes and clean clusters.

Privately judging if a node is good w.r.t. a set can be implemented easily by the Laplace mechanism. 
Then a natural strategy to prove the privacy is to apply advanced composition across all the iterations of the algorithm.
However, this only gives an $O(n^2)$ additive error, and the main technical contribution of the paper is a more sophisticated privacy accounting. Our key structural lemma says that any single node can be good w.r.t. neighborhoods of at most $\Tilde{O}(\Delta^*)$ different pivots.
Then, a careful argument shows that we only need to account for privacy loss for such iterations, which gives the desired bound. 
%then we can add some Laplace noise $\Lap(\Tilde{O}(\sqrt{\Delta^*}/\epsilon))$ and apply the (advanced) composition theorems to get the DP guarantee.
As for the utility proof, \cite{bbc04} observed that there exits a constant-approximation clustering $\OPT^{(0)}$ where each non-singleton cluster is clean.
We make a further observation that dissolving small clusters of size $\Tilde{O}(\sqrt{\Delta^*})$
can lead to an additive error of $\Tilde{O}(n\sqrt{\Delta^*})$.
% , and dissolving clusters between which and singletons there are large disagreements only increase the multiplicative error by a constant factor.
Denote the new clustering $\OPT^{(1)}\colon\C_1^{(1)},\cdots,\C_{t_1}^{(1)},S^{(1)}$, where each $\C_{i}^{(1)}$ is clean and has a large size,  and there are only small disagreements between $\C_{i}^{(1)}$ and $S^{(1)}$, where $S^{(1)}$ is the set of singletons.
The high-level intuition to prove the utility is that our algorithm can recover $\C_i^{(1)}$ well.
% and  our solution based on clustering of part of $\C_i^{(1)}$ and those singletons together should also be fine.
% or there are large disagreements between the cluster $\C_i^{(1)}$ and some singletons in $S^{(1)}$.

\subsection{More Related Work}
As mentioned earlier, Correlation clustering was first proposed by \cite{bbc04}, in which they also gave the first constant approximation for the minimization version and a PTAS for the maximization version, both for unweighted graphs.
%Due to its broad applicability in ML applications, the problem has been studied extensively.
The approximation of MinDis has been improved by subsequent works (\cite{acn08}), and the current best ratio is 2.06 by \cite{CMSY15}.
The problem has also been studied in various other settings, such as with fixed number of clusters \cite{gg05}, noisy or/and partial inputs \cite{ms10,MMV15}, and parallel computation \cite{ppo+15,clm+20}.

% Correlation clustering under DP constrained was first considered by \cite{bek21} along with some non-trivial upper bounds and an $\Omega(n\log n)$ additive lower bound.
% One can achieve the lower bound by the exponential mechanism, but it is inefficient and takes exponential time.
Finally, the Rank Aggregation problem is closely related to correlation clustering.
\cite{agkm21} consider Rank Aggregation problem under DP constraints, but their setting and techniques seem very different from ours.

\subsection{Outline}
In Section~\ref{sec:prel}, we give some basic definitions and backgrounds which are used throughout the work.
We present our main result for general graphs in Section \ref{sec:general_graph}.
We present our algorithm for the complete graphs, the privacy and utility analysis of our result in Section \ref{sec:unweigted}.
%  is deferred to Appendix due to space constraints. 
\section{Preliminaries}
\label{sec:prel}
%In particular, each edge is assigned fractional values $\omega_{ij}^+$ and $\omega_{ij}^-$ with the probability constraints (i.e. $\omega_{ij}^++\omega_{ij}^-=1$ for each edge).

\begin{definition}[Correlation-Clustering]
Let $G = (V,E)$
%$G=\{V,E,W\}$
be a weighted graph where $E=E^+\cup E^-$ is spitted into two disjoint subsets denoting the positive and negative labels of edges.
And for each edge $e\in E$, there is an associated non-negative weight $w_G(e)\geq 0$.
Given a clustering $\C=\{\C_1,\cdots,\C_k\}$, we say an edge $e\in E^+$ agrees with $\C$ if both endpoints of $e$ belong to the same cluster, and an edge $e\in E^-$ agrees with $\C$ if its both endpoints belong to different clusters.

%  the \textsf{agreement} $\agr(\C,G)$ between $\C$ and $G$ as the total weight of edges agreeing with $\C$, and 
We define the \textsf{disagreement} $\dis(\C,G)$ as the total weight of edges which do not agree with $\C$.
\end{definition}

\begin{definition}[Neighboring graphs]
Consider two weighted graphs $G,G'$ with the same node set and sign labels $\sigma,\sigma'\in\{-1,+1\}^{\binom{V}{2}}$. 
We say that $G$ and $G'$ are neighboring, if
\begin{align*}
    \sum_{e \in\binom{V}{2}}\left|\sigma_{e} w_G(e)-\sigma_{e}^{\prime} w_{G'}(e) \right| \leq 2.
\end{align*}
\end{definition}

\begin{definition}[Differential Privacy]
A (randomized) algorithm $\ALG$ is $(\epsilon,\delta)$-differentially private, if for any event $\mathcal{O}\in \mathrm{Range}(\ALG)$ and for any neighboring graphs $G,G'$ one has
\begin{align*}
    \Pr[\ALG(G)\in \mathcal{O}]\leq \exp(\epsilon)\Pr[\ALG(G')\in\mathcal{O}]+\delta.
\end{align*}
\end{definition}

% \begin{theorem}[Advanced Composition,Theorem 3.20 in \cite{dr14}]
% For all $\epsilon,\delta,\delta'\ge 0$, the $k$-fold adaptive composition of $(\epsilon,\delta)$-differentially private mechanism satisfies $(\epsilon',k\delta+\delta')$-differentia privacy for 
% \begin{align*}
%     \epsilon'=\sqrt{2k\log(1/\delta')}\epsilon+k\epsilon(e^{\epsilon}-1).
% \end{align*}
% \end{theorem}
% \Daogao{
% Need new theorems, as it is composition of $(\epsilon,\delta)$ and $(0,\delta)$ instead of homogeneous mechanisms. Maybe use Theorem 3.5 in \cite{kov15}}

\begin{theorem}[Theorem 3.5 in \cite{kov15}]
\label{thm:composition}
For any $\epsilon_{\ell}>0,\delta_{\ell}\in[0,1]$ for $\ell\in\{1,\cdots,k\}$ and $\Tilde{\delta}\in[0,1]$, the class of $(\epsilon_\ell,\delta_\ell)$-differentially private mechanism satisfy $(\Tilde{\epsilon}_{\Tilde{\delta}},1-(1-\Tilde{\delta})\Pi_{\ell=1}^{k}(1-\delta_\ell))$-differential privacy under $k$-fold adaptive composition,
where
%For $\Tilde{\epsilon}_{\Tilde{\delta}}=$
\begin{align*}
\begin{aligned}
\Tilde{\epsilon}_{\Tilde{\delta}}=
&\min \Big\{\sum_{\ell=1}^{k} \varepsilon_{\ell}, \sum_{\ell=1}^{k} \frac{\left(e^{\varepsilon_{\ell}}-1\right) \varepsilon_{\ell}}{e^{\varepsilon_{\ell}}+1}+\sqrt{\sum_{\ell=1}^{k} 2 \varepsilon_{\ell}^{2} \log \left(\frac{1}{\tilde{\delta}}\right)}, \\
&~~~\sum_{\ell=1}^{k} \frac{\left(e^{\varepsilon_{\ell}}-1\right) \varepsilon_{\ell}}{e^{\varepsilon_{\ell}}+1}+\sqrt{\sum_{\ell=1}^{k} 2 \varepsilon_{\ell}^{2} \log \left(e+\frac{\sqrt{\sum_{\ell=1}^{k} \varepsilon_{\ell}^{2}}}{\tilde{\delta}}\right)}
\Big\} 
\end{aligned}
\end{align*}

\end{theorem}

We refer to the Appendix~\ref{sec:more_prel} for more preliminaries, such as the basic composition, Laplace mechanism and some facts about Laplace distributions.
\section{General Graph}
\label{sec:general_graph}
%This section is basically based on \cite{bek21} by improving the LP.
%%%%%%
% ME: not really, we do not change the LP, we just add more computation
% in order to improve the bound
%%%%%%
%We present the algorithm and our main result of minimization version first.

In this section, we present our result for the general graphs.
Our improvement comes from strengthening the analysis of \cite{bek21}.
In nutshell, the DP mechanism of \cite{bek21} releases a synthetic graph $H$
which approximates the input graph $G$ in the cut distance.
They argue that the number of disagreements (and agreements)
of a fixed clustering consisting of $k$ clusters on $G$ and $H$
differ by at most $k$ times the respective cut distance bound.
Finally, they optimize $k$ to obtain the desired result.
We show that this factor $k$ is not necessary.

We define some notations before we state our results.
Given a graph $G$, for any subset $F\subseteq\binom{V}{2}$ of edges, we define $w_G(F):=\sum_{e\in F}w_G(e)$. And for two sets $S,T\subseteq V$ of nodes, we define $w_G(S,T):=\sum_{u\in S,v\in T}w_G((u,v))$.
For two (different) graphs $G$ and $H$ with the same node set $V$, we define the {\em cut distance} by
\begin{align*}
    d_{cut}(G,H)=\max_{S,T\subseteq V}|w_G(S,T)-w_H(S,T)|.
\end{align*}

We split $G$ into two disjoint sub-graphs $G^+$ and $G^-$ with the same node set, containing all positive and negative edges respectively. For example, if $e=(u,v)$ is labeled positive with weight $w_G(e)\geq 0$, then we have $w_{G^+}(e)=w_G(e)$ and $w_{G^-}(e)=0$.
And we have the following result.

% \Daogao{Define $G^+,G^-,d_{cut}$}
% Split the original graph $G$ into sub-graphs
% $G+$ and $G−$ on the same vertex set containing all positive and negative edges respect

%\begin{proposition}[\citet{bek21}, Lemma 20]
%\label{prop:bek21}
%Let $G$ be a complete graph with signed edges, and let $H$ be its release
%using the mechanism of \citet{bek21}. With high probability, we have
%\begin{align*}
%|w_H^+ (S,T) - w_G^+(S,T)| &\leq \tilde{O}(n^{3/2}/\epsilon) \text{ and}\\
%|w_H^- (S,T) - w_G^-(S,T)| &\leq \tilde{O}(n^{3/2}/\epsilon)
%\end{align*}
%for any $S, T \subseteq V$.
%\end{proposition}

%\citet{bek21} use Proposition~\ref{prop:bek21} to conclude that,
%for a fixed clustering $\gC$,
%the number of disagreements in the released graph $H$ differs from the number of disagreements in the original graph $G$
%by at most $\tilde{O}(k n^{3/2}/\epsilon)$, where $k$ is the number of clusters in $\gC$.
%We show how to use this proposition to devise a stronger bound.

\begin{lemma}
\label{lem:improving_bek21}
Let $G$ and $H$ be two graphs with signed edges such that
$d_{\cut}(G^+,H^+)\leq \beta$ and
$d_{\cut}(G^-, H^-) \leq \beta$, where  the graphs
$G^+, H^+$ and $G^-, H^-$ denote the induced graphs on positive
and negative edges respectively.
Then, for any clustering $\gC$, we have
\begin{align*}
|\dis(\gC, H) - \dis(\gC, G)| &\leq 6\beta.
% \text{ and}\\ |\agr(\gC, G) - \agr(\gC, H)| &\leq 6\beta.
\end{align*}
\end{lemma}

\begin{proof}
Let $\gC := \{C_1, \dotsc, C_k\}$ denote the clustering of the node set.
We have
\begin{align}
\dis(\gC, H) - \dis(\gC,G)
&= \sum_{i=1}^k (w_{H^-}(C_i, C_i) - w_{G^-}(C_i,C_i))
+ \sum_{(i,j): i\neq j} (w_{H^+}(C_i, C_j) - w_{G^+}(C_i, C_j))\label{eq:gen-dis}.
\end{align}
We show that absolute values of both sums can be bounded by a multiple of $\beta$.
Let's start with the term $\sum_{(i,j): i\neq j} (w_{H^+}(C_i, C_j) - w_{G^+}(C_i, C_j))$.
Let $I \cup J = [k]$ be a random partition, where each $i\in [k]$ is assigned either to $I$ or $J$ independently with equal probability.
Then, we have
\begin{align*}
\E\bigg[\sum_{i\in I, j\in J} (w_{H^+}(C_i, C_j) - w_{G^+}(C_i, C_j))\bigg]
    = \sum_{i\neq j} \frac12 (w_H^+(C_i, C_j) - w_G^+(C_i, C_j)),
\end{align*}
because each pair $i,j$ belong to different parts with probability $1/2$.
There must exist a partition $I^*, J^*$ such that
\begin{align*}
    &\frac{1}{2} \bigg|\sum_{i\neq j} (w_{H^+}(C_i, C_j) - w_{G^+}(C_i, C_j))\bigg|\\
\leq & \bigg|\sum_{i\in I^*, j\in J^*} (w_{H^+}(C_i, C_j) - w_{G^+}(C_i, C_j))\bigg|
    = |w_{H^+}(S, T) - w_{G^+}(S,T)|
\end{align*}

where $S = \bigcup_{i\in I^*} C_i$ and $T = \bigcup_{j\in J^*} C_j$.
Together with $d_{\cut}(H^+, G^+)\leq \beta$, this implies that
\begin{align}
\label{eq:bound_cluster_distance_positive}
    \sum_{i\neq j} (w_{H^+}(C_i, C_j) - w_{G^+}(C_i, C_j)) \leq 2\beta.
\end{align} 
%Without loss of generality, we assume that expressions in the preceding
%inequality are non-negative and choose $\{I^*, J^*\}$
%such that $\sum_{i\in I^*, j\in J^*} (w_H^+(C_i, C_j) - w_G^+(C_i, C_j))$
%is higher than the expectation.
%Then,
%setting $S = \bigcup_{i\in I^*} C_i$ and $T = \bigcup_{j\in J^*} C_j$, we get
%\begin{multline*}
%\frac12 \sum_{i\neq j} (w_H^+(C_i, C_j) - w_G^+(C_i, C_j))\\
%    \leq \sum_{i\in I, j\in J} (w_H^+(C_i, C_j) - w_G^+(C_i, C_j))\\
%    = w_H^+(S, T) - w_G^+(S,T) \leq \beta,
%\end{multline*}
%where the last inequality is is due to $d_{\cut}(H^+, G^+)\leq \beta$.

Now, consider the term $\sum_{i=1}^k (w_{H^-}(C_i, C_i) - w_{G^-}(C_i,C_i))$ in \eqref{eq:gen-dis}.
For each $i=1,\dotsc,k$, we consider a random partition
$C_i = A_i \cup B_i$ constructed by assigning each node
$v\in C_i$ independently either to $A_i$ or $B_i$ with equal probability.
Then, we have
\begin{align*}
\E\bigg[ \sum_{i=1}^k (w_{H^-}(A_i, B_i) - w_{G^-}(A_i, B_i))\bigg]
= \frac12 \sum_{i=1}^k (w_{H^-}(C_i, C_i) - w_{G^-}(C_i,C_i)).
\end{align*}
We choose sets $A_1^*, \dotsc, A_k^*, B_1^*, \dotsc, B_k^*$ which
make the absolute value of this expression higher than its expectation
and define two partitions of the node set $V$:
$\gP_1 = \{A_1^*\cup B_1^*, \dotsc, A_k^*\cup B_k^*\}$
and
$\gP_2 = \{A_1^*, \dotsc, A_k^*, B_1^*, \dotsc, B_k^*$\}.
Let $\gP_i(G)$ be the sum weights of violated edges crossing the partition
$\gP_i$ in graph $G$. One can verify easily that
$\gP_2(G^+) - \gP_1(G^+) = \sum_{i=1}^k w_{G^+}(A_i^*,B_i^*)$.
Therefore, we have
\begin{multline*}
\sum_{i=1}^k (w_{H^-}(A_i^*, B_i^*) - w_{G^-}(A_i^*,B_i^*))
= (\gP_2(H^-) - \gP_2(G^-)) - (\gP_1(H^-) - \gP_1(G^-))
\end{multline*}
Now, one of the following equations must hold:
\begin{align}
&|\gP_2(H^-) - \gP_2(G^-)|
\geq  \frac12 \bigg|\sum_{i=1}^k (w_{H^-}(A_i^*, B_i^*) - w_{G^-}(A_i^*,B_i^*))\bigg|
\label{eq:gen-case1}\\
&|\gP_1(H^-) - \gP_1(G^-)|
\geq  \frac12 \bigg|\sum_{i=1}^k (w_{H^-}(A_i^*, B_i^*) - w_{G^-}(A_i^*,B_i^*))\bigg|
\label{eq:gen-case2}
\end{align}

Case 1: Equation~(\ref{eq:gen-case1}) holds.
For each set $A_1^*, \dotsc, A_k^*$ and $B_1^*, \dotsc, B_k^*$,
we flip a fair coin and add all the nodes from that set either to $S$ or to $T$. Then, we have
\[
\E[ w_{H^-}(S,T) - w_{G^-}(S,T)] = \frac12 (\gP_2(H^-) - \gP_2(G^-))
\]

Case 2: Equation~(\ref{eq:gen-case2}) holds.
For each $i=1, \dotsc, k$, we flip a fair coin and
add all the nodes from $A_i\cup B_i$ either to $S$ or to $T$.
Then, we have
\[
\E[w_{H^-}(S,T) - w_{G^-}(S,T)] = \frac12 (\gP_1(H^-) - \gP_1(G^-))
\]

In both cases, our choice of sets $A_1^*,\cdots$, Equation~(\ref{eq:gen-case1}), Equation~(\ref{eq:gen-case2}), and
assumption that $d_{\cut}(H^-, G^-)\leq \beta$ imply
\begin{align}
    \bigg| \sum_{i=1}^k (w_{H^-}(C_i, C_i) - w_{G^-}(C_i,C_i))\bigg|
    \leq 4\beta.
    \label{eq:bound_cluster_distance_negative}
\end{align}

%In both cases, there must be a partition $V= S^* \cup T^*$ such that,
%by \eqref{eq:gen-case1} and \eqref{eq:gen-case2}, we have
%\[ \bigg| \sum_{i=1}^k (w_H^-(C_i, C_i) - w_G^-(C_i,C_i))\bigg|
%\leq 4|w_H^-(S^*,T^*) - w_G^-(S^*,T^*)|,
%\]
%which is at most $4\alpha$.
Now, the statement follows from Equation~(\ref{eq:gen-dis}), Equation~(\ref{eq:bound_cluster_distance_negative}) and  Equation~(\ref{eq:bound_cluster_distance_positive}).
We complete the proof.
% The proof for the number of agreements is analogous, and we omit it.
%\begin{align*}
%\agr(\gC, G) - \agr(\gC,H)
%&= \sum_{i=1}^k (w_G^+(C_i, C_i) - w_H^+(C_i,C_i))\\
%&+ \sum_{i\neq j} (w_G^-(C_i, C_j) - w_H^-(C_i, C_j))
%\end{align*}
%is analogous.
\end{proof}

% \begin{center}
%     \begin{align*}
%     \text{minimize } \sum_{(i,j)\in E^+}w_{ij}x_{ij}+\sum_{(i,j)\in E^-}w_{ij} x_{ij} \text{ s.t.}\\
%     x_{ij}+x_{jk}\leq x_{ik}, \text{ for all distinct } i,j,k\\
%     1\geq x_{ij}\geq 0, \text{ for all } i\neq j
% \end{align*}
% \end{center}

Lemma~\ref{lem:improving_bek21}, together with the following statement
from \cite{bek21}, implies 
% that there is an $\epsilon$-DP mechanism
% for release of complete unweighted graph which preserves numbers of
% disagreements and agreements of any clustering up to an additive term
% $O(n^{3/2}/\epsilon)$. Similarly, 
there is a $(\epsilon,\delta)$-DP mechanism for release of
weighted graphs which preserves number of disagreements and agreements
of any clustering up to an additive term
$O(\sqrt{\frac{mn}{\epsilon}}\log^2(\frac{n}{\delta}))$,
where $m$ denotes the total weight of the edges in the input graph.

% \begin{proposition}[\citet{bek21}, Lemma 20]
% \label{prop:bek21-complete}
% Let $G$ be a complete graph with signed edges, and let $H$ be its release
% using the $\epsilon$-DP mechanism of \citet{bek21}.
% With high probability, we have
% \begin{align*}
% d_{\cut}(H^+,G^+) &\leq \tilde{O}(n^{3/2}/\epsilon) \text{ and }\\
% d_{\cut}(H^-,G^-) &\leq \tilde{O}(n^{3/2}/\epsilon).
% %|w_H^+ (S,T) - w_G^+(S,T)| &\leq \tilde{O}(n^{3/2}/\epsilon) \text{ and}\\
% %|w_H^- (S,T) - w_G^-(S,T)| &\leq \tilde{O}(n^{3/2}/\epsilon)
% \end{align*}
% %for any $S, T \subseteq V$.
% \end{proposition}

\begin{proposition}[\citet{bek21} Section 4.2]
\label{prop:bek21-general}
Let $G$ be a general graph with weighted edges, which can be either positive or negative. 
Further we assume that the total value of weights is at most $m$.
% and the total value of negative weight edges is at most $m$.
Then there is an $(\epsilon,\delta)$-DP mechanism which releases synthetic graph $H$ satisfying:
\begin{align*}
\E[d_{\cut}(H^+,G^+)] &\leq
\textstyle
O(\sqrt{\frac{mn}{\epsilon}}\log^2\frac{n}{\delta}) \text{ and }\\
\E[d_{\cut}(H^-,G^-)] &\leq
\textstyle
O(\sqrt{\frac{mn}{\epsilon}}\log^2\frac{n}{\delta}).
\end{align*}
\end{proposition}

\begin{lemma}
Let $G$ be a general graph with weighted edges, which can be either positive or negative. 
Further we assume that the total value of weights is at most $m$.
% Further we assume that the total value of positive weight edges is at most $m$ and the total value of negative weight edges is at most $m$.
Then there is an $(\epsilon,\delta)$-DP algorithm to release a synthetic graph $H$ that satisfies for any clustering $\C$,
\begin{align*}
    |\dis(\C,H)-\dis(\C,G)|\leq O(\sqrt{\frac{mn}{\epsilon}}\log^2\frac{n}{\delta}).
\end{align*}
\end{lemma}
\begin{proof}
The proof follows from combining Lemma~\ref{lem:improving_bek21} and Proposition~\ref{prop:bek21-general}.
\end{proof}

Now we are ready to prove our main result for general weighted graphs.

\begin{theorem}
There is an $(\epsilon, \delta)$-DP algorithm for minimizing disagreements
on general weighted graphs and get a clustering $\C$ with the following guarantee:
\[ \textstyle
\dis(\C,G) \leq O(\log n) \dis(\OPT,G) + 
O(\sqrt{\frac{mn}{\epsilon}}\log^2(\frac{n}{\delta})).
\]
\end{theorem}

\begin{proof}
We use the previous lemma to construct a synthetic graph $H$.
On $H$, we can use any $\alpha$-approximation algorithm to find a clustering $\gC$.
Now consider,
\begin{align*}
\dis(\gC, H) &\leq \alpha \cdot \dis(\gC_H, H)
    \leq \alpha \cdot \dis(\OPT, H)\\
    &\leq \alpha \cdot \dis(\OPT, G) + \tilde{O}(\sqrt{mn/\epsilon})),
\end{align*}
where $\C_H$ is the optimal clustering with respect to $H$ and $\OPT$ is the optimal clustering with respect to $G$.
Further, note that $\dis(\gC,G) \leq \dis(\gC, H) + \tilde{O}(\sqrt{mn/\epsilon})$.
We get a clustering $\C$ such that
$\dis(\gC,G)\leq \alpha\dis(\OPT, G) + \tilde{O}(\sqrt{mn/\epsilon})$.

Finally, we can use the $O(\log n)$-approximation  algorithm from \cite{def+06} for the correlation clustering problem on weighted graphs, hence $\alpha = O(\log n)$,  which completes the proof.
\end{proof}
\section{Unweighted Graph}
\label{sec:unweigted}
In the MinDis problem on unweighted complete graphs, we assume all edges, either with positive or negative signs, have unit weights. 
That is $w_G(e)=1$ for any $e\in E$.

Before describing our algorithm, we make some definitions used in this section.
For any graph $G$, let $\Delta^*_G$ be the true maximum positive degree of all nodes on graph $G$.
Let $d_G(u)$ denote the positive degree of $u$ in graph $G$. 
If there is no confusion, we may use $\Delta^*$ and $d(u)$.
For a set $C\subseteq V$ of nodes, we denote $E^+(C)$ (resp. $E^-(C)$) to be the set of positive (resp. negative) edges with at least one endpoint in $C$, and $N^+(C)$ (resp. $N^-(C)$) to be the set of positive (resp. negative) neighboring nodes.
We use $\OPT$ to demonstrate the optimal clustering.
We may use $\ALG$ to represent either Algorithm~\ref{alg:alg} or the clustering output by Algorithm~\ref{alg:alg} for simplicity.

The main result of this section is the following:
\begin{theorem}
\label{thm:complete_main}
Given any unweighted complete graph $G=(V,E^+,E^-)$ and privacy parameters $\epsilon,\delta\in(0,1/2)$, Algorithm~\ref{alg:alg} is $(\epsilon,\delta)$-DP and outputs a clustering $\ALG$ such that
\begin{align*}
    \dis(\ALG,G)\leq O(1)\cdot \dis(\OPT,G)+O\left(\frac{n\log^4(n/\delta)}{\epsilon}\cdot\sqrt{\Delta^*+\frac{\log(n/\delta)}{\epsilon}}\right).
\end{align*}
\end{theorem}

\begin{algorithm}[]
   \caption{Algorithm $\ALG$ for complete graph}
   \label{alg:alg}
\begin{algorithmic}[1]
\STATE {\bf Input:} $G=(V,E^+,E^-)$
% \STATE {\bf Process:}
% \STATE Initialize the set of singleton $S=\{\}$
\STATE %Use $(\epsilon/10,0)$-DP Report Noisy Max and Laplace Mechanism and get estimation $\Delta_0$ of the maximum positive degree $\Delta^*:=\max_{u\in V}d_G(u)$
$\Delta_0:$ Use Noisy Max algorithm to privately estimate the maximum positive degree $\max_{u\in V}d_G(u)$
\label{ln:report_noisy_degree}
\STATE $\Delta\leftarrow \Delta_0+10\log (n/\delta)/\epsilon$\hfill \COMMENT{Prevent underestimation}
\STATE $c_l\leftarrow\lceil \Delta\rceil, k\leftarrow 0$, $\lambda\leftarrow 1/10,b_{\good}\leftarrow \sqrt{c_l}\log^2(n/\delta)$\label{ln:compare_sqrtn_end}
\WHILE{$V$ is not empty} \label{ln:while_loop_begin}
\STATE Pick an arbitrary node $v\in V$ as {\bf pivot}, $k\leftarrow k+1$ \label{ln:pick_pivot}
\STATE let $V\leftarrow V\setminus\{v\},E^+\leftarrow E^+\setminus E^+(v),E^-\leftarrow E^-\setminus E^-(v)$
\STATE $\Tilde{d}(v)\leftarrow \lceil d_G(v)+\Lap(10/\epsilon)\rceil$ \label{ln:judge_pivot_degree}
\IF{$\Tilde{d}(v)\leq 100 \sqrt{c_l}\log^4(n/\delta)/\epsilon$} \label{ln:if_degree_small}
\STATE Output $A_k\leftarrow \{v\}$ as a singleton
\STATE {\bf  Continue to Line~\ref{ln:pick_pivot} if $V$ is not empty}
\ENDIF
\STATE Let $B\leftarrow\{\}, t\leftarrow 2
\lceil\Tilde{d}(v)\rceil$
\FOR{each node $u_j\in V$}
\IF{$\PJudgeG(N^+(v),u_j,b_{\good},\lambda)$ is TRUE {\bf and} $t\ge0$} \label{ln:PJdugeG_part_one}
\STATE $B\leftarrow B\cup\{u_j\}$, $t\leftarrow t-1$
\ENDIF
\ENDFOR
\STATE Let $ \Tilde{|B|}\leftarrow  \lceil|B|+\Lap(10/\epsilon) \rceil$ \label{ln:judge_B_size}
\IF{$\Tilde{|B|}\le 9\Tilde{d}(v)/10$}
\STATE Output $A_k\leftarrow \{v\}$ as a singleton
\STATE {\bf  Continue to Line~\ref{ln:pick_pivot} if $V$ is not empty}
\ELSE
\STATE Let $t\leftarrow 2\Tilde{|B|},D\leftarrow \{\}$
\FOR{each node $u_j\in V\setminus B$}
\IF{$\PJudgeG(B,u_j,b_{good},4\lambda)$ is TRUE {\bf and} $t\ge 0$} \label{ln:PJudgeG_part_two}
\STATE $D\leftarrow D\cup \{u_j\}$, $t\leftarrow t-1$
\ENDIF
\ENDFOR
\STATE Let $A_k\leftarrow B\cup D$, output $A_k$ as a cluster,
\STATE $V\leftarrow V\setminus A_k,E^+\leftarrow E^+\setminus E^+(A_k), E^-\leftarrow E^-\setminus E^-(A_k)$
\ENDIF
\ENDWHILE
\STATE {\bf Output:} Clustering $\ALG$ (clusters and singletons)
\end{algorithmic}
\end{algorithm}

\begin{algorithm}[H]
\caption{$\PJudgeG$: Privately judge if a node $u$ is good with respect to a set $C$}
\label{alg:PJudgeG}
\begin{algorithmic}
\STATE {\bf Input:} Graph $G=(V,E^+,E^-)$, node $u$, set $C\subseteq V$, parameters $b_{\good},\lambda$
\IF{$|N^+(u)\cap C|+\Lap(2b_\good/\epsilon)\geq (1-\lambda)|C|$ and  $|N^+(u)\cap (V\setminus C)|\leq \lambda |C|+\Lap(2b_{\good}/\epsilon)$}
\STATE {\bf Return:} TRUE
\ELSE 
\STATE
{\bf Return:} FALSE
\ENDIF

\STATE
\end{algorithmic}
\end{algorithm}

%And we use $\kappa=\max\{\sqrt{n/\Delta^*},\sqrt{\Delta^*}\}$.

We prove the privacy and utility guarantees of Algorithm~\ref{alg:alg} separately.
The proof of privacy guarantee is presented in the following subsection, and we refer to the appendix for the proof of utility guarantee due to the limited space.

\subsection{Privacy Guarantee}
Now we consider the outputs of Algorithm~\ref{alg:alg} on two neighboring graphs $G$ and $G'$, which only differ by one fixed edge.
Let $(x,y)$ be this edge.
%We will fix the notations $x,y$ for the following privacy analysis.

The high-level idea to prove the privacy guarantee is to analyze the basic components used in the Algorithm~\ref{alg:alg} and then apply the composition theorems (Theorem~\ref{thm:basic_composition} and Theorem~\ref{thm:composition}).
Roughly speaking, a call to $\PJudgeG$ can lead to privacy loss.
We show that there are only $\Tilde{O}(c_l)$ ``dangerous'' calls to  the procedure that can lead privacy loss, each of which is
 $(\epsilon/(c_l\log(n/\delta)),\delta/n^4)$-DP.
 %steps is bounded by $\Tilde{O}(c_l)$, and the remaining steps, which are under $O(n)$-fold at most, can be interpreted by $(0,\delta/\poly(n))$-DP.
The remaining steps are $(0,\delta/\poly(n))$-DP and there can be at most polynomially many such steps.
Thus, the whole process is $(\epsilon,\delta)$-DP by composition.
Now we consider some basic components.

\begin{lemma}
\label{lm:noisy_max}
The Report Noisy Max and Laplace Mechanism (Line~\ref{ln:report_noisy_degree} in Algorithm~\ref{alg:alg}) is $(\epsilon/10,0)$-differentially private, and with probability at least $1-\delta/n^5$, we have $\Delta^*+15\log(n/\delta)/\epsilon\geq \Delta\geq \Delta^*+5\log(n/\delta)/\epsilon$.
%By post-processing, Line~\ref{ln:compare_sqrtn_begin} to Line~\ref{ln:compare_sqrtn_end} is DP.
\end{lemma}

\begin{lemma}
The Line~\ref{ln:judge_pivot_degree} and Line~\ref{ln:judge_B_size} in Algorithm~\ref{alg:alg} are $(\epsilon/10,0)$-DP.
With probability $1-\delta/n^5$, the estimation errors are at most $O(\log(n/\delta)/\epsilon)$.
\end{lemma}

The two lemmas above are classic results that follow directly from previous works \cite{dr14}.
In the following proof, we are conditioned on that $\Delta^*+15\log(n/\delta)/\epsilon\geq \Delta\geq \Delta^*+5\log(n/\delta)/\epsilon$.
Recall that we are considering two neighboring graphs $G,G'$ which differ on the sign of edge $(x,y)$.
It remains to bound the privacy loss due to $\PJudgeG$ at Line~\ref{ln:PJdugeG_part_one} (part-one) and at Line~\ref{ln:PJudgeG_part_two} (part-two).
For that, we define a concept which plays a crucial role in the following analysis.

\begin{definition}[hesitant]
Fix any $\lambda > 0$. For any node $u\in V$ and any set $S\subset V$, we say $u$ is $\lambda$-\hesi with respect to $V$ when the algorithm calls $\PJudgeG(S,u,b_{\good},\lambda)$, if $u$ and $S$ satisfy the following condition:
\begin{itemize}
    \item $|N^{+}(u)\cap S|>(1-\lambda)|S|- 10b_{\good}\log(n/\delta)/\epsilon$
    \item {\bf and} $|N^+(u)\cap \overline{S}|-\lambda |S|<10b_{\good} \log(n/\delta)/\epsilon$
\end{itemize}
\end{definition}

We consider the part-one of $\PJudgeG$  (Line~\ref{ln:PJdugeG_part_one}) first.
Obviously, we only need to take care of the part-one under two cases: (i) either $x$ or $y$ is the pivot, and we run $\PJudgeG$ with $N^+(x)$ or $N^+(y)$ as input parameters;
(ii) when $x$ or $y$ become the second parameters in the input of $\PJudgeG$.
A trivial analysis would suggest that  the total number of calls to  $\PJudgeG$ under these two cases is $O(n)$ and each call is $(\epsilon/(\sqrt{c_l}\log^2(n/\delta)),0)$-DP, which is not good enough to get the desired DP guarantee.
% and thus we can not get the desired DP guarantee under composition theorems.
This is where we invoke the concept of being \hesi.

\begin{lemma}
\label{lm:pjudgeg_pivotx}
A call to $\PJudgeG(N^+(x),u,b_{\good},\lambda)$ with a node $u\in V$ and a set $N^+(x)$ when $u$ is {\bf not} $\lambda$-\hesi w.r.t. $N^+(x)$ is $(0,\delta/n^4)$-DP.
\end{lemma}

\begin{proof}
As $u$ is not $\lambda$-\hesi with respect to $N^+(x)$, by the definition of being \hesi, we know either
$|N^+(u)\cap N^+(x)|\le (1-\lambda)|N^+(x)|-10b_{\good}\log(n/\delta)/\epsilon$ or $|N^+(u)\cap (V\setminus N^+(x))|-\lambda |N^+(x)|\ge 10b_{\good}\log(n/\delta)/\epsilon$.
Without loss of generality, we consider the first case. 

Recall the definition of DP, and let $P,P'$ denote the probability distributions with respect to neighboring inputs $G,G'$ respectively, we want to prove that
\begin{align}
\label{eq:dpproof_outputtrue}
    &P[\PJudgeG(N^+(x),u,b_{\good},\lambda)=\mathrm{TRUE}]\\ \nonumber
     & ~~~~~~\leq P'[\PJudgeG(N^+(x),u,b_{\good},\lambda)=\mathrm{TRUE}]+\delta/n^4
\end{align}
and 
\begin{align}
    \label{eq:epproof_outfalse}
    &P[\PJudgeG(N^+(x),u,b_{\good},\lambda)=\mathrm{FALSE}]\\ \nonumber
    & ~~~~~~\leq P'[\PJudgeG(N^+(x),u,b_{\good},\lambda)=\mathrm{FALSE}]+\delta/n^4.
\end{align}

By the concentration of Laplace distribution (Fact~\ref{ft:concentration_Lap}), it is true that 
\begin{align*}
    P[\PJudgeG(N^+(x),u,b_{\good},\lambda)=\mathrm{TRUE}]\leq \delta/n^4
\end{align*}
 and
 \begin{align*}
P'[\PJudgeG(N^+(x),u,b_{\good},\lambda)=\mathrm{TRUE}]\leq \delta/n^4  .
 \end{align*}
Thus Equation~(\ref{eq:dpproof_outputtrue}) holds direcly.
Next we prove Equation~(\ref{eq:epproof_outfalse}).

Let $X,Y\sim \Lap(2b_{\good}/\epsilon)$ be two independent Laplace random variables.
Let $T_1=(1-\lambda)|N^+(x)|-|N^+(u)\cap N^+(x)|$ and $T_2=|N^+(u)\cap (V\setminus N^+(x))|-\lambda |N^+(x)|$.

Then, 
\begin{align*}
    &P[\PJudgeG(N^+(x),u,b_{\good},\lambda)=\mathrm{FALSE}]
    = \Pr[X< T_1\cup Y<T_2].
\end{align*}
And we know 
\begin{align*}
    &P'[\PJudgeG(N^+(x),u,b_{\good},\lambda)=\mathrm{FALSE}]
    \geq  \Pr[X< T_1-1\cup Y<T_2-1].
\end{align*}

Recall that we are considering the case where $T_1\geq 10b_{\good}\log(n/\delta)/\epsilon$.
Hence  we know
\begin{align*}
    &~P[\PJudgeG(N^+(x),u,b_{\good},\lambda)=\mathrm{FALSE}]\\
    &~-P'[\PJudgeG(N^+(x),u,b_{\good},\lambda)=\mathrm{FALSE}]\\
    \leq&~ \Pr[X< T_1\cup Y<T_2]-\Pr[X< T_1-1\cup Y<T_2-1]\\
    =&~ \Pr[X\geq T_1-1]\Pr[Y\geq T_2-1]-\Pr[X\geq T_1]\Pr[Y\geq T_2]\\
    \leq &~ \Pr[X\geq T_1-1]\\
    \leq &~ \delta/n^4.
\end{align*}

The conclusion for the other case when $|N^+(u)\cap (V\setminus N^+(x))|-\lambda |N^+(x)|\ge \frac{10b_{\good}\log(n/\delta)}{\epsilon}$ follows by the same argument.
Thus we complete the proof.
\end{proof}

Using similar arguments, we can also prove the following lemma:
\begin{lemma}
\label{lm:DP-when-nonhesi}
A call to $\PJudgeG(S,x,b_{\good},\lambda)$ with node $x$ and any set $S\subset V$ as input when $x$ is not $\lambda$-\hesi is $(0,\delta/n^4)$-DP.
\end{lemma}

We continue the analysis of privacy.
Recall that we only need to take care of the calls to $\PJudgeG$ under two cases:  (i) either $x$ or $y$ is the pivot, and we run $\PJudgeG$ with $N^+(x)$ or $N^+(y)$ as input parameters;
(ii) when $x$ or $y$ become the second parameters in the input of $\PJudgeG$.
We bound the total number of times a node $u$ becomes \hesi under these two cases during the whole procedure of $\ALG$.

\begin{lemma}
\label{lm:bound_case1_hesi}
Suppose $x$ is chosen as the pivot for some iteration.
With probability at least $1-\delta/n^5$, the total number of times a node $u$ becomes  $\lambda$-\hesi with $N^+(x)$ is at most $2c_l$, i.e.
\begin{align*}
    \Pr[\sum_{u\in V}\Ind_{u \text{ is } \lambda\text{-} \hesi \text{ w.r.t. } N^+(x)}\leq 2c_l]\geq 1-\delta/n^5,
\end{align*}
and each such call to $\PJudgeG$ is $(\epsilon/(\sqrt{c_l}\log^2(n/\delta)),0)$-DP.
\end{lemma}
\begin{proof}
The DP guarantee of a single call to $\PJudgeG$ follows directly from the Laplace mechanism. 
Now we bound the total number of times a node $u$ becomes  $\lambda$-\hesi.

Consider the initial size of $|N^+(x)|$. 
If the size of $|N^+(x)|$ is smaller than $100\sqrt{c_l}\log^4(n/\delta)/\epsilon- 5\log(n/\delta)/\epsilon$, then with probability as least $1-\delta/n^5$, we will have $\Tilde{d}(x)\leq 100\sqrt{c_l}\log^4(n/\delta)/\epsilon$, and we will output $\{x\}$ as a singleton.
So we only need to focus on the case when $|N^+(x)|\geq 100\sqrt{c_l}\log^4(n/\delta)/\epsilon- 5\log(n/\delta)/\epsilon\geq 90\sqrt{c_l}\log^4(n/\delta)/\epsilon$.
% \Marek{should there be $5\log(n/\delta)$?} 
%where the last inequality follows from our setting of parameters.

% Recall that $|T_x(u_i)|>|N^+(x)|/10\geq 9\sqrt{c_l}\log^4(n/\delta)/\epsilon$.
% By the definition of being \hesi, we know that
% $|N^+(u_i)\cap T_x(u_i)|>(1-\lambda)|T_x(u_i)|-10b_{\good}\log(n/\delta)/\epsilon$ and $|N^+(u)\cap (V\setminus T_x(u_i))|<\lambda|T_x(u_i)|+10b_{\good}\log(n/\delta)/\epsilon$. 

Let $S\subset V$ be the set of nodes which are $\lambda$-\hesi w.r.t. $N^+(x)$.
For each node $u\in S$, we have $|N^+(u)\cap N^+(x)|>(1-\lambda)|N^+(x)|-10b_{\good}\log(n/\delta)/\epsilon>(1-2\lambda)|N^+(x)|$.
As we know $|E^+(N^+(x))|\leq c_l|N^+(x)|$, thus $|S|\leq \frac{c_l |N^+(x)|}{(1-2\lambda)|N^+(x)|}\leq 2c_l$.

\end{proof}

\begin{lemma}
\label{lm:bound_case2_hesi}
Consider the node $x$. 
Let $G_j$ denote the sub-graph induced on the remaining nodes when $\ALG$ selects the $j$-th pivot $v_j$.
With probability at least $1-\delta/n^4$, the total number of times $x$ becomes $\lambda$-\hesi w.r.t. some set $N^+_{G_j}(v_j)$ corresponding to pivot $v_j$ during the whole procedure is at most $O(c_l\log(n))$; that is, 
\begin{align*}
    \Pr[\sum_{j}\Ind_{x \text{ is }\hesi \text{ w.r.t.} N^+_{G_j}(v_j)}\leq O(c_l\log n)]\geq1-\delta/n^4.
\end{align*}
\end{lemma}

\begin{proof}
% Let $G_j$ denote the sub-graph induced on the remaining nodes when $\ALG$ selects the $j$-th pivot $v_j$.
In this notation, we have $G_1=G$.
First we consider the case when $d_{G_j}(x)\leq 50\sqrt{c_l}\log^4(n/\delta)/\epsilon$.
Suppose $x$ is $\lambda$-\hesi w.r.t. $N^+_{G_j}(v_j)$, which means that $|N^+_{G_j}(v_j)\cap N^+_{G_j}(x)|> (1-\lambda)|N^+_{G_j}(v_j)|-10b_{\good}\log(n/\delta)/\epsilon$ and $|N^+_{G_j}(x)\cap (V\setminus N^+_{G_j}(v_j))|-\lambda |N^+_{G_j}(v_j)|< 10b_{\good}\log(n/\delta)/\epsilon$.
Hence $d_{G_j}(v_j)=|N^+(v_j)|<\frac{d_{G_j}(x)+10b_{\good}\log(n/\delta)/\epsilon}{1-\lambda}\leq 90\sqrt{c_l}\log^4(n/\delta)/\epsilon$, which implies that with probability at least $1-\delta/n^9$, $v_j$ will be output as a singleton and $\ALG$ does not run $\PJudgeG$ on $x$ and $N^+_{G_j}(v_j)$.
Thus we should only consider the case when positive degree of $x$ is large.

Let the sequence of pivots selected by $\ALG$ be $\pi=\{v_1,\cdots,v_t\}$ before $x$ is deleted from the graph or is selected as the pivot.
If $x$ is the first pivot then we simply set $\pi=\emptyset$ and this lemma follows directly.
Let $\Event_j$ be the event that $v_j$ is the first pivot in $\pi$ such that $d_{G_j}(x)\leq 50\sqrt{c_l}\log^4(n/\delta)/\epsilon$.

Conditioned on $\Event_j$, we consider the total number of nodes $v_i$ for which $x$ is $\lambda$-\hesi w.r.t. $N^+_{G_i}(v_i)$ where $i<j$.
By the definition, if $x$ is $\lambda$-\hesi w.r.t. $N^+_{G_i}(v_i)$, then we know that $|N^+_{G_i}(x)\cap N^+_{G_i}(v_i)|>(1-\lambda)|N^+_{G_i}(v_i)|-10b_{\good}\log(n/\delta)/\epsilon$ and $|N^+_{G_i}(x)\cap (V\setminus N^+_{G_i}(v_i))|<\lambda |N^+_{G_i}(v_i)|+10b_{\good}\log(n/\delta)/\epsilon$.

For simplicity, we define $R_i:=|E^+_{G_i}(N^+_{G_i}(x))|$, where $N^+_{G_i}(x)$ is the positive neighborhood of $x$ in $G_i$ and $E^+_{G_i}(N^+_{G_i}(x))$ is the set of positive edges with at least one endpoint in $N^+_{G_i}(x)$.
Note that $R_{i+1}\leq R_{i}$.

Now we prove the following statement: if $x$ is $\lambda$-$\hesi$ w.r.t. $N^+_{G_i}(v_i)$, then $R_{i+1}\leq (1-\frac{1}{2c_l})R_i$.

By the assumption, we know that $d_{G_i}(x)>50\sqrt{c_l}\log^4(n/\delta)/\epsilon$.
Then if $x$ is $\lambda$-\hesi w.r.t. $N^+_{G_i}(v_i)$, we know $|N^+_{G_i}(x)\cap N^+_{G_i}(v_i)|>(1-\lambda)|N^+_{G_i}(v_i)|-10b_{\good}\log(n/\delta)/\epsilon$ and $|N^+_{G_i}(x)\cap (V\setminus N^+_{G_i}(v_i))|<\lambda |N^+_{G_i}(v_i)|+10b_{\good}\log(n/\delta)/\epsilon$, which implies that $(1-2\lambda)|N^+_{G_i}(v_i)|< (1-\lambda)|N^+_{G_i}(v_i)|-10b_{\good}\log(n/\delta)/\epsilon)< d_{G_i}(x)\leq (1+\lambda)|N^+_{G_i}(v_i)|+10b_{\good}\log(n/\delta)/\epsilon<(1+2\lambda)|N^+_{G_i}(v_i)|$.

For any node $z\in N^+_{G_i}(x)\cap N^+_{G_i}(v_i)$, we know that $(v_i,z),(z,x)\in E^+_{G_i}$, which implies that $(v_i,z)\in E^+_{G_i}(N^+_{G_i}(x))$.
Note that $v$ must be deleted in $G_{i+1}$, which leads to at least $\frac{1-2\lambda}{2(1+2\lambda)}d_{G_i}(x)$ deletions of edges in $E^+_{G_i}(N^+_{G_i}(x))$.
Then we know $R_{i+1}\leq R_{i}-\frac{1-2\lambda}{2(1+2\lambda)}d_{G_i}(x)\leq (1-\frac{1}{4c_l})R_i$ as $R_i\leq c_ld_{G_i}(x)$.

As we are conditioning on $\Event_j$, we have $R_{j-1}\ge |N^+_{G_{i-1}}(x)|\geq 50\sqrt{c_l}\log^4(n/\delta)/\epsilon$.
As $R_1\leq c_l^2$, we conclude that the total number of times $x$ becomes $\lambda$-\hesi is at most $O(c_l \log (n))$.

% As for the total time to be \hesi when $d(x)>99\sqrt{c_l}\log^2(n/\delta))$, we define $U$ to be the set $N^+(N^+(x))$, and denote $E_U$ to be the positive edges both of whose  endpoints belong to $U$.
% As for some node $u\in A(v)\cap N^+(x)$, either $(v,u)\in E^+$, or $(u,x)\in E^+$, $u$ is judged to be good with respect to $A(v)$ and thus will be clustered or outputted as a singleton.
% For each time $x$ is \hesi, then after this loop, the size of $E_U$ will decrease by a $(1-1/c_l)$ factor.
% Thus the total time will be bounded by $O(c_l\log n)$.

%The DP guarantee follows from the composition theorems (Advanced composition and basic composition).
\end{proof}

Combining Lemma~\ref{lm:DP-when-nonhesi}, Lemma~\ref{lm:bound_case1_hesi} and Lemma~\ref{lm:bound_case2_hesi} together, we can prove the DP guarantee of part-one $\PJudgeG$.
As for the part-two of $\PJudgeG$, we only need to consider the case when $x$ or $y$ is input as the single node of $\PJudgeG$. We prove the following.

\begin{lemma}
\label{lm:bound_x_hesi_B}
For the node $x$, the total number of times $x$ is $4\lambda$-\hesi w.r.t. some set $B$ during the whole procedure is at most $O(c_l\log n)$ with probability at least $1-\delta/n^5$.
\end{lemma}

The proof is essentially same as the one for Lemma~\ref{lm:bound_case2_hesi}.
Each time $x$ is $4\lambda$-\hesi w.r.t. $B$ means $|B\cap N^+_{G_i}(x)|$ is large and $B$ must be deleted, which means $|E^+_{G_{i+1}}(N^+_{G_{i+1}}(x))|\leq (1-\Omega(\frac{1}{c_l}))|E^+_{G_{i}}(N^+_{G_{i}}(x))|$.
Now we can complete the proof of the DP-guarantee.

\begin{theorem}
\label{thm:complete_privacy}
Given $0<\epsilon<1/2,0<\delta<1/2$,
Algorithm~\ref{alg:alg} is $(\epsilon,\delta)$-differentially private.
\end{theorem}

%As mentioned, we count the total steps needed for the compositions.

Combining the results above (Lemma~\ref{lm:noisy_max} to Lemma~\ref{lm:bound_x_hesi_B}) we know
with probability at least $1-\delta/n^4$, Algorithm~\ref{alg:alg} only needs two $(\epsilon/10,0)$-DP steps, $O(c_l\log(n))$ many $(\epsilon/(\sqrt{c_l}\log^2(n/\delta)),0)$-DP steps and $O(n)$ steps of $(0,\delta/n^4)$-DP sub-procedures.
The proof then follows from some basic calculations. %which we defer to the appendix.

\subsection{Utility Analysis}
Having proved the DP guarantee, now it suffices to prove the utility guarantee of our Algorithm \ref{alg:alg}.
% We prove our result by reduction.
Revisit some crucial concepts from \cite{bbc04}:
\begin{definition}[\cite{bbc04}]
We say a node $v$ is $\lambda$-good with respect to a set $C\subseteq V$, if it satisfies the following:
\begin{itemize}
    \item $|N^+(v)\cap C|\geq (1-\lambda)|C|$
    \item $|N^+(v)\cap (V\setminus C)|\leq \lambda |C|$
\end{itemize}
A set $C$ is $\eta$-clean if all $v\in C$ are $\eta$-good w.r.t. C.
\end{definition}

As mentioned before, \cite{bbc04} made a key observation that there is a clustering with clean clusters and a constant approximation.
\begin{lemma}[Lemma 6 in \cite{bbc04}]
\label{lm:clean_clustering}
For $0<\eta<1$, there exists a clustering $\OPT^{(0)}$ for graph $G$ in which each non-singleton cluster is $\eta$-clean and 
\begin{align*}
    \dis(\OPT^{(0)},G)\leq (\frac{9}{\eta^2}+1)\dis(\OPT,G).
\end{align*}
\end{lemma}

Given a graph $G$,
for a (possibly random) set $A$ of nodes and any (possibly random) clustering $\C$, we define $\cost(A,\C,G)$ to be the (expected) cost related to nodes in $A$ under the clustering $\C$.
To be more clear, we cluster all nodes in $G$ according to the clustering $\C$ and count for violated edges which have at least one endpoint in $A$, that is the total number of negative edges in $E^-(A)$ inside clusters plus the total number of positive edges in $E^+(A)$ between clusters, under clustering $\C$.
Moreover, for a set $A\subseteq V$ of nodes, we let $G\setminus A$ be the sub-graph deduced by $V\setminus A$, that is we delete the nodes in $A$ and the edges (whatever positive or negative) connected with at least one node in $A$.

Fix $\eta=\lambda/10=1/100$ in the following proof. 
Suppose the clustering in the Lemma~\ref{lm:clean_clustering} is $\OPT^{(0)}:\C_1^{(0)},\C_2^{(0)},\cdots,\C_{t_0}^{(0)},\{u\}_{u\in S^{(0)}}$ where $S^{(0)}$ is the set of singletons.
We define 
\begin{align}
\label{def:opt_i}
    \OPT^{(1)}\leftarrow\Cleanup(G,\OPT^{(0)},110\sqrt{c_l}\log^4(n/\delta)/\epsilon)
\end{align} 
to be the clustering outputted by the Algorithm $\Cleanup$ (Algorithm~\ref{alg:clean-up}).
We denote the new clustering by $\OPT^{(1)}:\C_1^{(1)},\C_2^{(1)},\cdots,\C_{t_1}^{(1)},\{u\}_{u\in S^{(1)}}$ (remove those empty-sets).
The algorithm~$\Cleanup$ and clustering $\OPT^{(i)}$ ($i\in\{0,1\}$) are only defined for our utility proof, and we do not need to know the specific $\OPT^{(i)}$ and never need to run the algorithm~$\Cleanup$.

% Now we define some clustering recursively with respect to the (random) output of Algorithm~\ref{alg:alg}. 

\begin{algorithm}
\caption{Algorithm CleanUp}
\label{alg:clean-up}
\begin{algorithmic}[1]
\STATE {\bf Input:} A graph $G$; the clustering $\C:\C_1,\cdots, \C_t$ and the set $S$ of singletons; parameters $T$

\STATE {\bf Process:}
\FOR{$i=1,\cdots,t$}
\IF{$|\C_i|\leq T$}\label{ln:dis_case_one}
\STATE Dissolve the cluster, $S\leftarrow S\cup \{\C_i\},\C_i'=\emptyset$
\STATE {\bf Continue} 
\ENDIF
\ENDFOR
\STATE {\bf Output:} The new clustering $\C':\C_1',\cdots,\C_t'$ and the set $S$ of singletons
\end{algorithmic}
\end{algorithm}

% For each $\C$ which is a cluster of clustering $\OPT^{(0)}$ restricted on $G\setminus \cup_{t=1}^{i-1}A_t$, 
% (i) if size of $\C$ is no more than $100\sqrt{c_l}\log^4(n/\delta)/\epsilon$, dissolve it;
% (ii) else if $\cost(\C,\OPT^{(0)},G\setminus \cup_{t=1}^{i-1}A_t)\ge \Omega(|\C|^2)$, dissolve it;
% (iii) run Algorithm~Procedure $\eta$-Clean-Up in \cite{bbc04} (Algorithm~\ref{alg:clean-up}) on it.

The high-level idea is to show $\OPT^{(1)}$ is a good clustering (Equation~\eqref{eq:dis_OPT_1}) with some good properties, and Algorithm~\ref{alg:alg} can recover each non-singleton cluster in $\OPT^{(1)}$ well.
We analyze the Algorithm $\Cleanup$ first and try to build Equation~\eqref{eq:dis_OPT_1}.

We define $D_1$ as follows
\begin{align}
\label{eq:define_Di}
    D_1:=&\dis(\OPT^{(1)},G)-\dis(\OPT^{(0)},G)
\end{align}
to capture the loss occurred by Algorithm $\Cleanup$.
We can prove the following claim:
\begin{claim}
\label{clm:bound_D_1}
Running $\Cleanup(G,\OPT^{(0)},110\sqrt{c_l}\log^4(n/\delta)/\epsilon)$,
we have
\begin{align}
\label{eq:bound_Di_by_Mi}
    D_1\le & O\Big(n\cdot\sqrt{c_l}\log^4(n/\delta)/\epsilon\Big).
\end{align}
\end{claim}
\begin{proof}

Recall that we denote the clustering w.r.t. $\OPT^{(0)}$ by $\C_1^{(0)},\cdots,C_{t_0}^{(0)},\{u\}_{u\in S^{(0)}}$.
Denote the set of nodes $M$ which are not singletons in $\OPT^{(0)}$ but become singletons in $\OPT^{(1)}$.
We denote $N_{j}:=M\cap \C_j^{(0)}$ for each $j\in[t_0]$ and rewrite $D_1$ as follows:
\begin{align}
\label{eq:rewriteD_1}
    D_1=&\sum_{j=1}^{t_0}\big(\omega_{G^+}(N_{j},\C_j^{(0)})-\omega_{G^-}(N_{j},\C_j^{(0)})\big).
\end{align}
If $\Cleanup$ dissolves $\C_j^{(0)}$, then $|\C_j^{(0)}|\leq 110\sqrt{c_l}\log^4(n/\delta)/\epsilon $) and
$N_{j}=\C_j^{(0)}$, we know that $\omega_{G^+}(N_{j},\C_{j}^{(0)})-\omega_{G^-}(N_{j},\C_j^{(0)})\leq |N_{j}|^2/2\leq O(1)|N_{j}|\cdot\sqrt{c_l}\log^4(n/\delta)/\epsilon$.
By Equation~\eqref{eq:rewriteD_1} we know $$D_1\leq \sum_{j=1}^{t_0}O(1)|N_j|\cdot \sqrt{c_l}\log^4(n/\delta)/\epsilon\leq O(n\cdot\sqrt{c_l}\log^4(n/\delta)/\epsilon).$$
Hence we prove Equation~(\ref{eq:bound_Di_by_Mi}).
\end{proof}

By the definition of $D_1$ (Equation~\eqref{eq:define_Di}) and Claim~\ref{clm:bound_D_1}, we have
\begin{align}
\label{eq:dis_OPT_1}
    \dis(\OPT^{(1)},G)\leq \dis(\OPT^{(0)},G)+O(n\sqrt{c_l}\log^4(n/\delta)/\epsilon).
\end{align}

We also need the following lemma, which follows immediately from the definitions:
\begin{lemma}
\label{lm:clean_small_cost}
For any graph $G$ and any clustering $\C=\C_1,\cdots,\C_t,\{u\}_{u\in S}$.
If non-singleton cluster $\C_i$ is $\eta$-clean, then
\begin{align*}
    \cost(\C_i,\C,G)\leq \eta|\C_i|^2.
\end{align*}
\end{lemma}

Consider the clustering $\OPT^{(1)}:\C_1^{(1)},\C_2^{(1)},\cdots,\C_{t_1}^{(1)},\{u\}_{u\in S^{(1)}}$.
We know for each non-singleton cluster $\C_i^{(1)}$ is $\eta$-clean and thus $\cost(\C_i^{(1)},\OPT^{(1)},G)\leq \eta|\C_i^{(1)}|^2$ by Lemma~\ref{lm:clean_small_cost}, and has size at least $110\sqrt{c_l}\log^4(n/\delta)/\epsilon$.
Having demonstrated the properties of $\OPT^{(1)}$,
as mentioned before, it suffices to show with high probability, Algorithm~\ref{alg:alg} can recover each non-singleton cluster in $\OPT^{(1)}$ well.
Let $A_i$ be the (random) set of nodes outputted by Algorithm~\ref{alg:alg} as either a cluster or a singleton in $i$-th iteration ($i$th pivot), where for the initialization we set $A_0=\emptyset$.
Note that there are $n$ nodes in the graph. 
If for some $j<n$ Algorithm~\ref{alg:alg} finishes the clustering and $\cup_{i=1}^{j}A_i=V$, we define $A_i=\emptyset$ for $j+1\leq i \leq n$.
We have the following two lemmas:

% \begin{restatable}{lemma}{lemAone}\label{lm:observation_A_1}
% With probability at least $1-\delta/n^5$, either $A_1\subset S^{(1)}$, or $\exists i$ such that $\C_i^{(1)} \subset A_1 \subset \C_i^{(1)}\cup S^{(1)}$.
% \end{restatable}

\begin{lemma}
\label{lm:observation_A_1}
With probability at least $1-\delta/n^5$, either $A_1\subset S^{(1)}$, or $\exists i$ such that $\C_i^{(1)} \subset A_1 \subset \C_i^{(1)}\cup S^{(1)}$.
\end{lemma}

\begin{lemma}
\label{lm:base_lemma}
For any graph $G=(V,E^+,E^-)$ and any clustering $\C=\C_1,\C_2,\cdots,\C_t,\{u\}_{u \in S}$ for $V$, if $|V|\leq n$, any non-singleton cluster $\C_i$ in $\C$ is $\eta$-clean and thus $\cost(\C_i,\C,G)\leq \eta|\C_i|^2$,
then we have
\begin{align}
\label{eq:bound_error_withA1}
    &\cost(A_1,\ALG,G)
    \leq O(1)\cost(A_1,\C,G)
    + O(\E[|A_1|]\sqrt{c_l}\log^4(n/\delta)/\epsilon),
\end{align}
where $A_1$ is the (random) output (either a cluster or a singleton) of $\ALG$ for the first pivot,
and the expectation is taken over randomness coins of $\ALG$.
\end{lemma}

% Do this recursively, and we can prove the utility of our algorithm $\Cleanup$.
% We have the following claims:

% \begin{claim}
% For $i$th iteration where $i\geq 1$, we will dissolve at most one cluster.
% That's $t_{i+1}\ge t_{i}-1$ for any $i\geq 1$.
% \end{claim}
Utility guarantee of $\ALG$ can be bounded recursively by the lemmas above.

We assume Lemma~\ref{lm:observation_A_1} and Lemma~\ref{lm:base_lemma} hold first and finish our main result on utility, and refer to the Appendix for the omitted proof.

\begin{theorem}
\label{thm:complete_utility}
The utility of the Algorithm~\ref{alg:alg} satisfies
\begin{align*}
\dis(\ALG,G)\leq O(1)\cdot \dis(\OPT,G)+O\left(\frac{n\log^4(n/\delta)}{\epsilon}\cdot\sqrt{\Delta^*+\frac{\log(n/\delta)}{\epsilon}}\right).
\end{align*}
\end{theorem}

\begin{proof}
% Note that $\cost(A_{n},\OPT^{(n)},G\setminus\cup_{j=1}^{n-1}A_j)=\dis(\OPT^{(n)},G\setminus\cup_{j=1}^{n-1}A_j)$ and recall the definition of $D_1$ (Equation~(\ref{eq:define_Di})).

Note that for any $i\ge 1$, by Lemma~\ref{lm:observation_A_1}, we know that with probability at least $1-\delta/n^4$, any (non-singleton) cluster $\C_j^{(1)}$ on sub-graph $G\setminus \cup_{t=1}^{i-1}A_t$ has a size no smaller than $110\sqrt{c_l}\log^4(n/\delta)/\epsilon$, is $\eta$-clean and satisfies that $\cost(\C_j^{(1)},\OPT^{(1)},G\setminus \cup_{t=1}^{i-1}A_t)\leq \eta |\C_j^{(i)}|^2$.
Conditional on the above events, the preconditions in Lemma~\ref{lm:base_lemma} hold and thus for any $i\ge 1$, we have
\begin{align}
\label{eq:bound_Ai}
    &\cost(A_i,\ALG,G\setminus \cup_{t=1}^{i-1}A_t)
    \leq O(1)\cost(A_i,\OPT^{(1)},G\setminus \cup_{t=1}^{i-1}A_t)
    +O\left(\E[|A_i|]\sqrt{c_l}\log^4(n/\delta)/\epsilon\right).
\end{align}
% \Marek{there is a condition in l3.12 about cost(Ci,C,G) being small. how do you deal with the case where it is big?}

Hence we know that
\begin{align*}
    \dis(\ALG,G)=& \sum_{i=1}^{n}\cost(A_i,\ALG,G\setminus \cup_{t=1}^{i-1}A_t)\\
    \leq & \sum_{i=1}^{n}O(1)\cost(A_i,\OPT^{(1)},G\setminus \cup_{t=1}^{i-1}A_t) +\sum_{i=1}^{n}O(\E[|A_i|]\sqrt{c_l}\log(n/\delta)/\epsilon)\\
    \leq & \sum_{i=1}^{n}O(1)\cost(A_i,\OPT^{(1)},G\setminus \cup_{t=1}^{i-1}A_t) +O(n\sqrt{c_l}\log(n/\delta)/\epsilon)\\
    \leq & O(1)\dis(\OPT^{(1)},G)+O(n\sqrt{c_l}\log(n/\delta)/\epsilon)\\
    \leq & O(1)\dis(\OPT^{(0)},G)+O(n\sqrt{c_l}\log^4(n/\delta)/\epsilon)\\
    \leq & O(1)\dis(\OPT,G)+O(n\sqrt{c_l}\log^4(n/\delta)/\epsilon),
\end{align*}
where the first line follows from the definition, the second line follows from Equation~(\ref{eq:bound_Ai}), the third line follows from that $A_i$ and $A_j$ are disjoint and there are at most $n$ nodes in the graph, the forth line follows from the recursive relationships and definitions, the fifth line follows from Equation~(\ref{eq:dis_OPT_1}) and the last line follows from Lemma~\ref{lm:clean_clustering}.

We know $\E[c_l-\Delta^*]\leq O(\log(n/\delta)/\epsilon)$, and complete the proof.
\end{proof}

Combining Theorem~\ref{thm:complete_privacy} and Theorem~\ref{thm:complete_utility}, we complete the proof of our main result Theorem~\ref{thm:complete_main}.

\section*{Acknowledgment}
The author would like to thank Marek Eliáš and Janardhan Kulkarni for many helpful discussions on the project and comments on improving the presentation.

\addcontentsline{toc}{section}{References}
\bibliographystyle{plainnat}
\bibliography{ref}
\appendix
% \section{Clean-Up algorithm and its proof}
% \begin{algorithm}
% \caption{Procedure $\eta$-Clean-Up}
% \label{alg:clean-up}
% \begin{algorithmic}
% \STATE {\bf Input:} A clustering: $\C_1,\cdots, \C_t$ and the set $S$ of singletons; Parameters $\eta,\beta>0$

% \STATE {\bf Process:}
% \FOR{$i=1,\cdots,t$}
% \STATE Let $D_i\subseteq \C_i$ be the set of $\frac{\eta}{3}$-bad vertices w.r.t. $\C_i$
% \IF{$|D_i|\geq \frac{\beta}{3}|\C_i|$}
% \STATE Dissolve the cluster, and $S\leftarrow S\cup \C_i,\C_i'=\emptyset$
% \ELSE \STATE Let $S\cup D_i$, and $\C_i'\leftarrow \C_i\setminus D_i$
% \ENDIF
% \ENDFOR
% \STATE {\bf Output:} The new clustering $\C_1',\cdots,\C_t'$ and the set $S$ of singletons
% \end{algorithmic}
% \end{algorithm}
% \section{Proof of Theorem~\ref{thm:complete_privacy}}

\newpage
\section{More Preliminaries}
\label{sec:more_prel}
\begin{theorem}[Basic Composition, \cite{dmns06}]
\label{thm:basic_composition}
Given $k$ mechanisms and suppose mechanism $\ALG_i$ is $(\epsilon_i,\delta_i)$-differentially private, then this class of mechanism satisfy $(\sum_{i=1}^{k}\epsilon_i,\sum_{i=1}^{k}\delta_i)$-differentially private under $k$-fold composition.
\end{theorem}
\begin{definition}[The Laplace Distribution]
The probability density function of Laplace distribution $\Lap(\mu,b)$ is
\begin{align*}
    \begin{gathered}
f(x \mid \mu, b)=\frac{1}{2 b} \exp \left(-\frac{|x-\mu|}{b}\right)
=\frac{1}{2 b} \begin{cases}\exp \left(-\frac{\mu-x}{b}\right) & \text { if } x<\mu \\
\exp \left(-\frac{x-\mu}{b}\right) & \text { if } x \geq \mu\end{cases}
\end{gathered}
\end{align*}
In this work, we write $\Lap(b)$ to denote the Laplace distribution with zero mean and scale $b$, and denote a random variable $X\sim \Lap(b)$.
\end{definition}

\begin{fact}
\label{ft:concentration_Lap}
If $X\sim \Lap(b)$, then $\E[|X|^2]=2b^2$ and
\begin{align*}
    \Pr[|X|\geq tb]=e^{-t}.
\end{align*}
\end{fact}

\begin{lemma}[Laplace Mechanism]
Given any function $f:\Xi\rightarrow \R^k$ where for any neighboring datasets $\D,\D'\in \Xi$, $\|f(\D)-f(\D')\|_1\leq \Delta f$.
The Laplace mechanism is outputting $f(\D)+(Y_1,\cdots,Y_k)$ where $Y_i$ are i.i.d. random variables drawn from $\Lap(\Delta f/\epsilon)$.
The Laplace mechanism is $(\epsilon,0)$-DP.
\end{lemma}

\section{Omitted Proof}
As graph $G$ is fixed, we may omit $G$ in the notations ``$\cost()$'' in the following proof.

\subsection{Proof of Lemma~\ref{lm:observation_A_1}}
\newtheorem*{lm:observation_A_1}{Lemma~\ref{lm:observation_A_1}}
\begin{lm:observation_A_1}
With probability at least $1-\delta/n^5$, either $A_1\subset S^{(1)}$, or $\exists i$ such that $\C_i^{(1)} \subset A_1 \subset \C_i^{(1)}\cup S^{(1)}$.
\end{lm:observation_A_1}

\begin{proof}
We need the following statement, which follows immediately from the definitions:
\begin{lemma}
\label{lm:phase_two_othercleansets}
Let $C$ be an $\eta$-clean set of size at least $100\sqrt{c_l}\log^4(n/\delta)/\epsilon$.
For any set $A$ such that $C\cap A=\emptyset$, we know for any node $u\in C$, $u$ is not $4\lambda$-\hesi w.r.t. $A$.
\end{lemma}

Recall $\ALG$ uses $O(n)$ Laplace random variables with respect to the first pivot.
We denote these $O(n)$ random variables by set $\mathrm{RL}$.
By the concentration of Laplace distribution (Fact~\ref{ft:concentration_Lap}) and union bound, we can argue that, with probability at least $1-\delta/n^5 $, for each Laplace random variable $X\in \mathrm{RL}$, we have $|X|\leq 6\log(\delta/n)\sqrt{\E[|X|^2]}$.
Denote this event by $\Event_{\mathrm{RL}}$ and immediately we have
\begin{align*}
    \Pr[\Event_{\mathrm{RL}}]\geq 1-\delta/n^5.
\end{align*}

It suffices to prove conditional on $\Event_{\mathrm{RL}}$, either $A_1\subset S^{(1)}$, or $\exists i$ such that $\C_i^{(1)} \subset A_1 \subset \C_i^{(1)}\cup S^{(1)}$.
By the definition of \hesi, we have the following claim directly:
\begin{claim}
\label{cl:nothesi_notgood}
Conditional on $\Event_{\mathrm{RL}}$, if for some node $u\in V$ and some set $C$ where $\ALG$ runs $\PJudgeG(C,u,b_{\good},\lambda)$ during the process and $u$ is not $\lambda$-\hesi with respect to $C$, then
running sub-procedure $\PJudgeG(C,u,b_{\good},\lambda)$ returns FALSE.
\end{claim}

Basically, we consider the different possible cases over the universe of all possible outputs of $A_1$.
In general we write $A_1=B\cup D$, where $B$ and $D$ represent the set appended into $A_1$ in the part-one and part-two respectively.
If $A_1=B$ is a singleton, then we have $D=\emptyset$.
For simplicity, in the following argument, we use $\C_i$ and $S$ to denote $\C_i^{(1)}$ for $i\in[t_1]$ and $S^{(1)}$ respectively.
We do category analysis and demonstrate that all those cases violating Lemma~\ref{lm:observation_A_1} are impossible conditional on $\Event_{\mathrm{RL}}$.

{\bf Case (1)}: Some node $v$ in the non-singleton cluster is selected as the pivot.
Without loss of generality, we assume the pivot $v\in\C_1$.
We divide Case(1) further based on whether $A_1$ is a cluster or a singleton.

{\bf Sub-Case(1.1)}: $A_1$ is a cluster.
In this Sub-Case, we know that $|\C_1|\geq 110\sqrt{c_l}\log^4(n/\delta)/\epsilon$ and $d(v)\geq (1-\eta)110\sqrt{c_l}\log^4(n/\delta)/\epsilon$, $(2+\eta)d(v) \geq |B|\geq \frac{4d(v)}{5}$ and $|D|\leq (2+\eta)|B|$.

We prove the following statement first:
$A_1\cap \C_i=\emptyset$ for $\forall i\neq 1$.

As $\C_1$ is $\eta$-clean, then we know $|N^+(v)\cap \C_1|\geq (1-\eta)|\C_1|$ and $|N^+(v)\cap (V\setminus \C_1)|\le \eta |\C_1|$.
For any node $z\in\C_i$ where $i\neq 1$, we know $|z\cap N^+(v)|\leq \eta|\C_1|\leq \frac{\eta}{1-\eta}|N^+(v)|\leq (1-\lambda)|N^+(v)|-10b_{\good}\log(n/\delta)/\epsilon$, which means that $z$ is not appended into the set $B$ conditional on $\Event_{\mathrm{RL}}$.

For any $z\in \C_i$ where $i\neq 1$, we also know $\C_i$ is $\eta$-clean and thus $|N^+(z)\cap (V\setminus\C_i)|\leq \eta|\C_i|$, and thus we know $|N^+(z)\cap B|\leq \eta |\C_i|$ and $|N^+(z)\cap (V\setminus B)|\geq (1-\eta)|\C_i|$.
Either $|B|\ge |\C_i|$ or $|B|<|\C_i|$, we know $z$ is not appended into the set $D$.
Thus we prove the statement.

Consider the situation when $\C_1\not\subset A_1$.
Then we know some some node $u\in \C_1$ is not appended in either $B$ or $D$ and thus $u\notin A_1$.

Basically, we know for any node $u\in \C_1$, we have $|N^+(v)\cap N^+(u)|\geq (1-2\eta)|\C_1|\geq \frac{1-2\eta}{1+\eta}|N^+(v)|$ and $|N^+(u)\setminus N^+(v)|\leq 2\eta |\C_1|\leq \frac{2\eta}{1-\eta}|N^+(v)|$,
which means running the sub-procedure $\PJudgeG(N^+(v),u,b_{\good},\lambda)$ returns TRUE.

The only possibility is the size of the set of nodes which are good w.r.t. $N^+(v)$ is too large.
In this case we know $|B\setminus \C_1|\geq (2-\eta)d_v-|\C_1|\geq \frac{1-2\eta}{1+\eta}|\C_1|$.
By the analysis in the situation above, we know for any node $u\in B\setminus\C_1$, one has $|N^+(u)\cap \C_1|\geq (1-3\lambda)|\C_1|$,
which means $\cost(\C_1,\OPT^{(1)},G)\ge (1-3\lambda)|\C_1|\times |B\setminus\C_1|\geq 
\frac{(1-3\lambda)(1-2\eta)}{1+\eta}|\C_1|^2
$ and thus violates the precondition.
So this situation is impossible.

{\bf Sub-Case (1.2)}: $A_1$ is a singleton.
For any node $u\in\C_1$, by the analysis above, we know running sub-procedure $\PJudgeG(N^+(v),u,b_{\good},\lambda)$ returns TRUE, which means all nodes in $\C_1$ can be appended into set $B$ if the size of $B$ does not violate the constraint.
And $|\C_1|\geq d(v)/(1+\eta)$, which means $\ALG$ does not dissolve $B$ due to its small size and the $\ALG$ must output a cluster. 
Thus this Sub-Case is impossible.

{\bf Case (2)}: Some node $v\in S$ is selected as the pivot, $A_1$ is a cluster, and $A_1\cap (\cup_{j=1}^{t}\C_j)\neq\emptyset$.
Recall we know $|N^+(v)|\geq (1-\eta)110\sqrt{c_l}\log^4(n/\delta)/\epsilon$.

One can argue that situation when $B\cap(\cup_{j=1}^{t}\C_j)=\emptyset$ is impossible by Lemma~\ref{lm:phase_two_othercleansets}.
If $B\cap(\cup_{j=1}^{t}\C_j)=\emptyset$ then we have $D\cap(\cup_{j=1}^{t}\C_j)=\emptyset  $ and thus $A_1\cap (\cup_{j=1}^{t}\C_j)=\emptyset$, which is contradiction.

Without loss of generality, assume $u\in \C_1$ is the first node in $\cup_{j=1}^{t}\C_j$ to be appended into $B$.
We prove $A_1\cap (\cup_{j=1}^{t}\C_j)\subset \C_1$ under this assumption.

If $u$ is appended into $B$, then $u$ must be  $\lambda$-\hesi w.r.t. $N^+(v)$, which means that
$|N^+(v)\cap N^+(u)|>(1-\lambda)|N^+(v)|-10b_{\good}\log(n/\delta)/\epsilon\ge (1-\lambda-\eta)|N^+(v)|$ and $|N^+(u)\setminus N^+(v)|<\lambda |N^+(v)|+10b_{\good}\log(n/\delta)/\epsilon<(\lambda+\eta)|N^+(v)|$.
Consider any node $z\in \C_2$, we now argue $z\notin A_1$.
Recall that both $\C_1$ and $\C_2$ are $\eta$-clean.
Thus $|N^+(u)\cap \C_1|\ge (1-\eta)|\C_1|, |N^+(u)\setminus \C_1|\le \eta |\C_1|, |N^+(z)\cap \C_2|\ge (1-\eta)|\C_2|$ and $|N^+(z)\setminus \C_2|\le \eta|\C_2|$.
Note that $ (1-\lambda-\eta)|N^+(v)|  \le |N^+(u)|\leq (1+\lambda+\eta)|N^+(v)|$.
Also we know $|N^+(v)\cap \C_1|\geq (1-\lambda-\eta)|N^+(v)|-\eta|\C_1|\geq (1-\lambda-\eta-\frac{(1+\lambda+\eta)\eta}{1-\eta})|N^+(v)|$ and $|N^+(v)\setminus \C_1|\leq (\lambda+\eta)|N^+(v)|+ \eta|\C_1|\leq (\lambda+3\eta)|N^+(v)|$.

Hence we know that for node $z\in \C_2$, if we want $z$ to be $\lambda$-\hesi w.r.t. $N^+(v)$, we need $\frac{1-\lambda}{1+\eta}|N^+(v)| \leq |\C_2|\leq \frac{1+\lambda}{1-\eta}|N^+(v)|$.
We have $|N^+(z)\cap N^+(v)|=|N^+(z)\cap N^+(v) \cap \C_1|+|N^+(z)\cap N^+(v) \cap (V\setminus \C_1)|\leq \eta |\C_2| + (\lambda+3\eta)|N^+(v)|\leq (\lambda+5\eta)|N^+(v)|$
Then whatever the size of $|\C_2|$ is, we know $z$ is not $\lambda$-\hesi w.r.t. $N^+(v)$ and is not appended into $B$.
If $u\in \C_2$, then $u\notin B$.
As $B\cap\C_2=\emptyset$ and thus $A_1\cap \C_2=\emptyset$ by Lemma~\ref{lm:phase_two_othercleansets}.
The same argument holds for other clusters, so we prove $A_1\cap (\cup_{j=1}^{t}\C_j)\subset \C_1$.

Now we consider the following two situations:\\
{\bf Situation (i):} $|B\cap \C_1| \geq 9 |B\setminus \C_1|$. 
At first, we prove that if $|B\cap \C_1| \geq 9 |B\setminus \C_1|$, then for node $u\in \C_1\setminus B$, $\PJudgeG(B,u,b_{\good},4\lambda)$ outputs TRUE.

First, we know that $(1+\eta)|N^+(v)|\ge|B|$, $\frac{1-\lambda}{1+\eta}|N^+(v)|\leq |\C_1|\leq \frac{1+\lambda}{1-\eta}|N^+(v)|$.
And we know that $(1+\eta)|\C_1|\ge|B|\geq \frac{9}{10}|N^+(v)|$, thus we know that $\frac{(1+\eta)^2}{1-\lambda}|\C_1|\ge|B|\ge \frac{9(1-\eta)}{10(1+\lambda)}|\C_1|$, which means that $ |B\cap\C_1|\geq \frac{81(1-\eta)}{100(1+\lambda)}|\C_1|$ and $|B\setminus \C_1|\leq \frac{(1+\eta)^2}{10(1-\lambda)}|\C_1|$.

For any node $u\in \C_1\setminus B$, we know $| N^+(u)\cap B|\geq (1-\eta+\frac{81(1-\eta)}{100(1+\lambda)}-1) |\C_1|\geq (1-3\lambda)|\C_1|\geq (\frac{(1-3\lambda)(1-\lambda)}{(1+\eta)^2})|B|$ and $| N^+(u)\setminus B|=|(N^+(u)\cap \C_1)\setminus B|+|(N^+(u)\setminus \C_1)\setminus B|\leq (1-\frac{81(1-\eta)}{100(1+\lambda)}+\eta) |\C_1|\leq 3\lambda |B|$.
Hence we know $u$ is judged $4\lambda$-good w.r.t. $B$.

If $\C_1\not\subset A_1$, we know there are too many nodes which are judged $4\lambda$-good w.r.t. $B$ and $\ALG$ does not append all nodes in $\C_1$ into $D$.
In particular, for any $z\in D\setminus \C_1$, we know $|N^+(z)\cap \C_1|\geq |N^+(z)\cap \C_1\cap B|\geq  (1-\lambda-\eta)|B|-3\lambda|\C_1| \geq |\C_1|/2$.
And we know $|D\setminus\C_1|\geq |\C_1|$, which means under these conditions and assumptions,
$\cost(\C_1,\OPT^{(1)})\geq |D\setminus\C_1|\cdot |\C_1|/2\geq |\C_1|^2/2$, violating the precondition that $\cost(\C_1,\OPT^{(1)})\leq \eta |\C_1|^2/2$  and is impossible.
Then we know $\C_1\subset A_1 $ in this situation.

{\bf Situation (ii):} $|B\cap \C_1| < 9 |B\setminus \C_1|$.

For any node $z\in B\setminus \C_1$, we know
$|N^+(z)\cap N^+(v)|\geq |N^+(z)\cap N^+(v)\cap \C_1|\geq (1-\lambda-\eta)|N^+(v)|- (\lambda+3\eta)|N^+(v)|= (1-2\lambda-4\eta)|N^+(v)|\geq \frac{(1-2\lambda-4\eta)(1-\eta)}{1+\lambda}|\C_1|$.
%$d(z)\geq (1-\lambda-\eta)d(v)\geq \frac{(1-\lambda-\eta)d(u)}{1+\lambda+\eta}\geq \frac{(1-\lambda-\eta)(1-\eta)}{1+\lambda+\eta}|\C_1|$.
We know $|B|\ge \frac{4}{5}|N^+(v)|\geq \frac{4(1-\eta)}{5(1+\lambda)}|\C_1|$.
Hence we know for this particular $\C_1$, we know $\cost(\C_1,\OPT^{(1)},G)\geq |B\setminus \C_1|\cdot(1-4\lambda)|\C_1|\geq (1-4\lambda)^2|\C_1|^2$, violating the precondition.
Thus conditional on $\Event_{\RL}$, we know this situation is impossible.

Combining the arguments of all cases and situations together, we know either $A_1\subset S$ or $\C_1\subset A_1\subset \C_1\cup S$ conditional on $\Event_{\mathrm{RL}}$.
\end{proof}

\subsection{Proof of Lemma~\ref{lm:base_lemma}}
% \Daogao{Check the ratios. Maybe put it into the appendix.}

\newtheorem*{lm:base_lemma}{Lemma~\ref{lm:base_lemma}}
\begin{lm:base_lemma}
For any graph $G=(V,E^+,E^-)$ and any clustering $\C=\C_1,\C_2,\cdots,\C_t,\{u\}_{u \in S}$ for $V$, if $|V|\leq n$, any non-singleton cluster $\C_i$ in $\C$ is $\eta$-clean, and $\cost(\C_i,\C,G)\leq \eta|\C_i|^2$,
then we have
\begin{align}
    &\cost(A_1,\ALG,G)
    \leq O(1)\cost(A_1,\C,G)
    + O(\E[|A_1|]\sqrt{c_l}\log^4(n/\delta)/\epsilon),
\end{align}
where $A_1$ is the (random) output (either a cluster or a singleton) of $\ALG$ for the first pivot,
and the expectation is taken over randomness coins of $\ALG$.
\end{lm:base_lemma}

\begin{proof}
We recapture the definition of $\Event_{\mathrm{RL}}$.
$\ALG$ uses $O(n)$ Laplace random variables at most during the procedure, and we denote this $O(n)$ random variables by set $\mathrm{RL}$.
By the concentration of Laplace distribution (Fact~\ref{ft:concentration_Lap}) and union bound, we can argue that, with probability at least $1-\delta/n^5 $, for each Laplace random variable $X\in \mathrm{RL}$, we have $|X|\leq 6\log(\delta/n)\sqrt{\E[|X|^2]}$.
Denote this event by $\Event_{\mathrm{RL}}$ which satisfies $\Pr[\Event_{\mathrm{RL}}]\geq 1-\delta/n^5$.

Let $\cost(A_1,\ALG,G\mid \Event_{\mathrm{RL}})$ be the expected cost conditional on $\Event_{\mathrm{RL}}$, then we know
\begin{align*}
    &\cost(A_1,\ALG,G)\\
    =&
    \cost(A_1,\ALG,G\mid \Event_{\mathrm{RL}})\Pr[\Event_{\mathrm{RL}}]+\cost(A_1,\ALG,G\mid \neg\Event_{\mathrm{RL}})\Pr[\neg\Event_{\mathrm{RL}}]\\
    \leq & \cost(A_1,\ALG,G\mid \Event_{\mathrm{RL}})\Pr[\Event_{\mathrm{RL}}]+ \delta/n^3.
\end{align*}
Then in order to prove Lemma~\ref{lm:base_lemma}, it suffices to prove 
\begin{align}
\label{eq:conditional_bound_error_withA1}
    \cost(A_1,\ALG,G\mid \Event_{\mathrm{RL}})\leq O(1)\cost(A_1,\C,G\mid \Event_{\mathrm{RL}})+O(\E[|A_1|\mid \Event_{\mathrm{RL}}]\sqrt{c_l}\log^4(n/\delta)/\epsilon).
\end{align}

% Without loss of generality, we can assume $i=1$ and to prove 
% \begin{align}
% \label{eq:induciton_base}
%     &\cost(A_1,\ALG,G)\\\nonumber
%     &\leq O(1)\cost(A_1,\OPT^{(1)},G)\\\nonumber
%     &+O(|A_1|\sqrt{c_l}\log/\epsilon)
% \end{align}

% Recall that $\OPT^{(1)}:\C_1,\cdots,\C_{t_1}^{(1)},\{v\}_{v\in S^{(1)}}$, where $\C_j^{(1)}$ for any $j$ is of size at least $100\sqrt{c_l}\log/\epsilon$ and $\eta$-clean. \Daogao{Say $\eta<\lambda/10$}

Our following proof is conditional on $\Event_{\mathrm{RL}}$.
Basically, we consider the different possible cases over the universe $\Omega$ of all possible outputs of $A_1$, do case analysis and show Equation~(\ref{eq:bound_error_withA1}) holds conditional on all of different cases.
In general we write $A_1=B\cup D$, where $B$ and $D$ represent the cluster of the part-one and part-two respectively.
If $A_1=B$ is a singleton, then we know $D=\emptyset$.
Recall that the benchmark clustering $\C:\C_1,\C_2,\cdots,\C_t,\{u\}_{u\in S}$ in the statement of Lemma~\ref{lm:base_lemma}.

{\bf Case (1)}, denoted by $\Omega_1$: Some node $v$ in the non-singleton cluster is selected as the pivot.

Without loss of generality, we assume the pivot $v\in\C_1$.
By the proof of Lemma~\ref{lm:observation_A_1}, we know $A_1$ is a cluster and $\C_1\subset A_1\subset \C_1\cup S$.
% {\bf Sub-Case(1.1)}, denoted by $\Omega_{1}$: $A_1$ is a cluster.
% We have argued that it is impossible that $\C_1\not\subset A_1$ in the proof of Lemma~\ref{lm:observation_A_1}.
% So we know $\C_1\subseteq A_1$.

In this case, for any node $u\in B\setminus\C_1$, one has $|N^+(u)\cap \C_1|\ge |N^+(u)\cap \C_1\cap N^+(v)|\geq (1-\eta)|\C_1|-2\lambda|N^+(v)|\geq (1-\eta-2\lambda(1+\eta))|\C_1|\geq (1-3\lambda)|\C_1|$.
And for any node $u\in D\setminus\C_1$, we know $d(u)\geq |B|/2\geq 2d(v)/5\geq |\C_1|/5$.
 Hence we have $\cost(A_1,\C\mid \Event_\RL,\Omega_{1})\geq (1-3\lambda)|C_1|\cdot |B\setminus\C_1|+\frac{|\C_1|\times|D\setminus\C_1|}{5}\geq |\C_1|\cdot |A_1\setminus\C_1|/5$.  Note that $\cost(A_1,\ALG \mid \Event_\RL,\Omega_{1})\leq \cost(A_1,\C \mid \Event_{\RL},\Omega_{1})+|\C_1|\cdot|A_1\setminus \C_1|+|A_1\setminus \C_1|^2=O(1)\cost(A_1,\C \mid \Event_\RL,\Omega_{1}) $ as $|\C_1|\ge \Omega(|A_1\setminus \C_1|)$.

Combining these together, we know conditional on $\Event_\RL$, we know $\C_1\subset A_1$ and 
\begin{align}
\label{eq:omega_1}
    \cost(A_1,\ALG \mid \Event_{\RL},\Omega_1)
    =& O(1)\cost(A_1,\C \mid \Event_{\RL},\Omega_{1}).
\end{align}
% Equation~(\ref{eq:conditional_bound_error_withA1}) holds conditional on Case (1).

% \Daogao{So in this case we can dissolve $\C_1$ directly.
% But the problem is how to trace it.
% Maybe just modify the definition of $\OPT^{(i)}$.
% }

{\bf Case (2)}, denoted by $\Omega_2$: Some node $v\in S$ is selected as the pivot.

We need to divide this case further.

{\bf Sub-Case (2.1)}, denoted by $\Omega_{2.1}$: $A_1$ is a singleton.
This Sub-Case is fine as one has 
\begin{align}
    \cost(A_1,\ALG \mid \Event_{\RL},\Omega_{2.1})=\cost(A_1,\C \mid \Event_{\RL},\Omega_{2.1})
    \label{eq:omeage_2.1}
\end{align}
immediately as $A_1$ is singleton in both $\ALG$ and $\C$.

{\bf Sub-Case (2.2)}, denoted by $\Omega_{2.2}$: $A_1$ is a cluster, and  $A_1\cap(\cup_{j=1}^{t}\C_j)=\emptyset$.

In this Sub-Case we know $d(v)\ge 99\sqrt{c_l}\log^4(n/\delta)/\epsilon$, or $v$ is outputted as a singleton.

Under this Sub-Case, we know $A_1\subset S$ and thus 
\begin{align*}
    \cost(A_1,\ALG \mid \Event_{\RL},\Omega_{2.2})\leq \cost(A_1,\C \mid \Event_{\RL},\Omega_{2.2})+ \E\Big[|A_1|^2\mid \Event_{\RL},\Omega_{2.2}\Big]/2.
\end{align*}

Consider two situations separately:\\
{\bf Situation (i):} $|A_1|\leq 100\sqrt{c_l}\log^4(n/\delta)/\epsilon$, denoted by $\Omega_{2.21}$.
% We know $\cost(A_1,\ALG\mid\Event_{\RL},\Omega_{2.21} )-\cost(A_1,\C\mid \Event_{\RL},\Omega_{2.21})\leq \E[|A_1|^2\mid \Event_{\RL},\Omega_{2.21}]$ and 
Hence 
\begin{align}
\label{eq:omega_2.21}
    \cost(A_1,\ALG\mid \Event_{\RL},\Omega_{2.21})\leq& \cost(A_1,\C\mid \Event_{\RL},\Omega_{2.21})\nonumber\\
    & + O(\E[|A_1|\cdot\sqrt{c_l}\log^4(n/\delta)/\epsilon\mid \Event_{\RL},\Omega_{2.21}])
\end{align}
holds immediately.

{\bf Situation (ii):} $|A_1|> 100\sqrt{c_l}\log^4(n/\delta)/\epsilon$, denoted by $\Omega_{2.22}$.
We know 
$4|N^+(v)|/5\leq |A_1|\leq (4+\eta)|N^+(v)|$,
and for any node $u\in A_1$ one has
$d(u)\ge (1-5\lambda)|N^+(v)|$.
Thus we know
$\cost(A_1,\ALG\mid \Event_{\RL},\Omega_{2.22})\leq\E[|A_1|^2/2 \mid \Event_{\RL},\Omega_{2.22}]$, and $\cost(A_1,\C\mid \Event_{\RL},\Omega_{2.22})\ge \E[|A_1|\cdot (1-5\lambda)|N^+(v)|\mid \Event_{\RL},\Omega_{2.22}]\geq \E[2|A_1|^2/5\mid \Event_{\RL},\Omega_{2.22}]$. 
Hence we have the following Equation
\begin{align}
\label{eq:omega_2.22}
    \cost(A_1,\ALG\mid \Event_{\RL},\Omega_{2.22})\leq O(1)\cost(A_1,\C\mid \Event_{\RL},\Omega_{2.22}).
\end{align}

% Then we know $\cost(A_1,\C,G)\geq \sum_{u\in F}d(u)/2 \geq \Omega(|A_1|^2)$ thus Equation~\ref{eq:bound_error_withA1} holds in this Sub-Case, as $\cost(A_1,\ALG,G)\leq O(|A_1|^2)$ and $\cost(A_1,\OPT^{(1)},G)\geq \Omega(|A_1|^2)$.

%Sub-Case(2.3): If some node $v_i\in \C_1$ is privately judged good and output $A_1$ as singleton.

{\bf Sub-Case(2.3)}, denoted by $\Omega_{2.3}$: $A_1$ is a cluster, and $A_1\cap (\cup_{j=1}^{t}\C_j)\neq\emptyset$.

Without loss of generality, assume $u\in \C_1$ is the first node in $\cup_{j=1}^{t}\C_j$ to be appended into $B$.
Then we know $\C_1\subset A_1$ and $|B\cap \C_1| > 9 |B\setminus \C_1|$ by the proof of Lemma~\ref{lm:phase_two_othercleansets}.
% Now we consider the following two situations:\\
% {\bf Situation (i):} $|B\cap \C_1| \geq 9 |B\setminus \C_1|$, denoted by $\Omega_{2.31}$. 

% Particularly, if $|B\cap \C_1| \geq 9 |B\setminus \C_1|$, then for node $u\in \C_1\setminus B$, $\PJudgeG(B,u,b_{\good},4\lambda)$ outputs TRUE.
% First, we know that $(1+\eta)|N^+(v)|\ge|B|$, $\frac{1-\lambda}{1+\eta}|N^+(v)|\leq |\C_1|\leq \frac{1+\lambda}{1-\eta}|N^+(v)|$.
% And we know that $(1+\eta)|\C_1|\ge|B|\geq \frac{9}{10}|N^+(v)|$, thus we know that $\frac{(1+\eta)^2}{1-\lambda}|\C_1|\ge|B|\ge \frac{9(1-\eta)}{10(1+\lambda)}|\C_1|$, which means that $ |B\cap\C_1|\geq \frac{81(1-\eta)}{100(1+\lambda)}|\C_1|$ and $|B\setminus \C_1|\leq \frac{(1+\eta)^2}{10(1-\lambda)}|\C_1|$.
% Combining the fact that $\C_1$ is $\eta$-clean, we can prove the argument above.

% For any node $u\in \C_1\setminus B$, we know $| N^+(u)\cap B|\geq (1-\eta+\frac{81(1-\eta)}{100(1+\lambda)}-1) |\C_1|\geq (1-3\lambda)|\C_1|\geq (\frac{(1-3\lambda)(1-\lambda)}{(1+\eta)^2})|B|$ and $| N^+(u)\setminus B|=|(N^+(u)\cap \C_1)\setminus B|+|(N^+(u)\setminus \C_1)\setminus B|\leq (1-\frac{81(1-\eta)}{100(1+\lambda)}+\eta) |\C_1|\leq 3\lambda |B|$.
% Hence we know $u$ is judged $4\lambda$-good w.r.t. $B$.

Note that $|A_1|=\Theta(|\C_1|),\cost(A_1,\ALG\mid \Event_{\RL},\Omega_{2.3})-\cost(A_1,\C\mid \Event_{\RL},\Omega_{2.3})\leq O\Big(\E\big[|\C_1|\cdot|A_1\setminus \C_1|+|A_1\setminus \C_1|^2\mid \Event_{\RL},\Omega_{2.3}\big]\Big)$, while $\cost(A_1,\C\mid \Event_{\RL},\Omega_{2.3})\geq \Omega\Big(\E\big[|\C_1|\cdot|A_1\setminus \C_1|\mid \Event_{\RL},\Omega_{2.3}\big]\Big)$. Hence
\begin{align}
\label{eq:omega_2.3}
    \cost(A_1,\ALG\mid \Event_{\RL},\Omega_{2.3})\leq O(1)\cost(A_1,\C\mid \Event_{\RL},\Omega_{2.3}).
\end{align}

% {\bf Situation (ii):} $|B\cap \C_1| < 9 |B\setminus \C_1|$, denoted by $\Omega_{2.32}$.

% For any node $z\in B\setminus \C_1$, we know
% $|N^+(z)\cap N^+(v)|\geq |N^+(z)\cap N^+(v)\cap \C_1|\geq (1-\lambda-\eta)|N^+(v)|- (\lambda+3\eta)|N^+(v)|= (1-2\lambda-4\eta)|N^+(v)|\geq \frac{(1-2\lambda-4\eta)(1-\eta)}{1+\lambda}|\C_1|$.
%$d(z)\geq (1-\lambda-\eta)d(v)\geq \frac{(1-\lambda-\eta)d(u)}{1+\lambda+\eta}\geq \frac{(1-\lambda-\eta)(1-\eta)}{1+\lambda+\eta}|\C_1|$.
% We know $|B|\ge \frac{4}{5}|N^+(v)|\geq \frac{4(1-\eta)}{5(1+\lambda)}|\C_1|$.
% Hence we know for this particular $\C_1$, we know $\cost(\C_1,\C,G)\geq |B\setminus \C_1|*(1-4\lambda)|\C_1|\geq (1-4\lambda)^2|\C_1|^2$, violating the precondition.
% Thus conditional on $\Event_{\RL}$, we know this situation is impossible and $\Omega_{2.32}=\emptyset$.

% If we know $\cost(A_1,\OPT^{(1)},G)\geq \Omega(|A_1|^2)\geq \Omega(|\C_1|^2)$ and $\cost(A_1,\ALG,G)-\cost(A_1,\OPT^{(1)},G)\leq O(1)|A_1||C_1|+O(|A_1|^2)\leq O(1)\cost(A_1,\OPT^{(1)},G)$.
% \Marek{I do not see this. what if $C_1\subset A_1$ and $v$ is the only vertex in $A_1\setminus C_1$?}
% Combining these together, we know Equation~(\ref{eq:bound_error_withA1}) holds in this Sub-Case.

Combining Equations~\eqref{eq:omega_1} to \eqref{eq:omega_2.3} together, we prove Equation~(\ref{eq:conditional_bound_error_withA1}) and complete the proof.
\end{proof}

\end{document}